\documentclass[pdflatex]{sn-jnl}

\usepackage{graphicx}%
\usepackage{multirow}%
\usepackage{amsmath,amssymb,amsfonts}%
\usepackage{amsthm}%
\usepackage{mathrsfs}%
\usepackage[title]{appendix}%
\usepackage{xcolor}%
\usepackage{textcomp}%
\usepackage{manyfoot}%
\usepackage{booktabs}%
\usepackage{algorithm}%
\usepackage{algorithmicx}%
\usepackage{algpseudocode}%
\usepackage{listings}%
\usepackage{biblatex}
\usepackage{amsmath}
\usepackage{amssymb}
\usepackage{graphicx}
\usepackage{pgfplots}
\pgfplotsset{compat=1.18}
\usepackage{subfigure}
\usepackage{url}
\usepackage{times}
\usepackage{mathtools}
\usepackage{tikz}
\usetikzlibrary{shapes.misc}
\usepackage{pgfplots}
\usepackage{clipboard}
\usetikzlibrary{datavisualization}
\usepackage{diagbox}
\newcommand\figref{Figure~\ref}
\newcommand{\R}{\mathbb{R}}
\newcommand{\norm}[1]{\left\lVert#1\right\rVert}
\DeclareMathOperator{\rank}{rank}

\theoremstyle{plain}%
\newtheorem{theorem}{Theorem}
\newtheorem{proposition}[theorem]{Proposition}%
\newtheorem{lemma}[theorem]{Lemma}%
\newtheorem{corollary}[theorem]{Corollay}%
\newtheorem{example}{Example}%

\theoremstyle{remark}%
\newtheorem{remark}{Remark}%
\newtheorem{definition}{Definition}%

\begin{document}

\title{On the Convergence of Gradient Descent for Large Learning Rates}


\author*[1]{\fnm{Alexandru} \sur{Cr\u aciun}}\email{acraciun@tum.de}

\author[1]{\fnm{Debarghya} \sur{Ghoshdastidar}}\email{ghoshdas@cit.tum.de}

\affil[1]{\orgname{Technical University of Munich}}


\abstract{
A vast literature on convergence guarantees for gradient descent and derived
methods exists at the moment. However, a simple practical situation remains
unexplored: when a fixed step-size is used, can we expect gradient descent to
converge starting from any initialization? We provide fundamental
impossibility results showing that convergence becomes impossible no matter
the initialization if the step-size gets too big. Looking at the asymptotic
value of the gradient norm along the optimization trajectory, we see 
a sharp transition as the step-size crosses a critical value. This
has been observed by practitioners, yet the true mechanisms through which
it happens remain unclear beyond heuristics. Using results from
the theory of dynamical systems, we provide a proof of this in the case of linear
neural networks with a squared loss. We also prove the impossibility of
convergence for more general losses without requiring strong assumptions such as
Lipschitz continuity for the gradient. We validate our findings through
experiments with non-linear networks.
}


\maketitle

\section{Introduction}
Conventional analysis of the convergence of gradient-based algorithms has mostly
been carried out under two assumptions: the loss function $L$ has Lipschitz
continuous gradient and the step-size along the optimization trajectory $\eta$
is strictly smaller than the inverse of the greatest eigenvalue of the Hessian
of the loss at that point, commonly referred to as the \emph{sharpness}. As long
as these conditions hold, the ``Descent Lemma'' \parencite[Lemma
1.2.3]{nesterov2018} guarantees that the loss will decrease at each iteration of
gradient descent: 
\begin{equation} 
  \label{eq:descent_lemma}
  L(\theta_{k+1}) \leq L(\theta_k) - \eta\left(1-K\frac{\eta}{2}\right)\norm{\nabla L(\theta_k)}^2,
\end{equation} 
where $K$ is the Lipschitz constant of the gradient. Under these assumptions,
the previous equation shows that the loss converges to a fixed value if and only
if the gradient descent trajectory converges to a critical point of the loss.
For models used in practice, both the assumption of Lipschitz continuity of the
gradient and the bound on the step-size must be dropped, thus many analyses on the convergence
of gradient descent no longer hold. Motivated by the need of understanding 
models used in practice, we formulate the following question:

\textbf{Given an initial value $\theta_0$ and a fixed step-size $\eta
> 0$ will the iterates of gradient descent converge?}

Interest in this question has been increasing due to recent empirical studies
which have shown that even though $\eta$ does not satisfy the assumptions of the
Descent Lemma,  gradient descent still manages to decrease the loss, albeit in a
non-monotonic fashion, and stabilizes in a trajectory where the sharpness hovers
around the value of $2/\eta$. This phenomenon has been called optimization in
the Edge of Stability regime \parencite{cohen2022,cohen2022b,ahn2022}. While 
trying to explain why Edge of Stability happens specifically when training neural networks,
\textcite[Takeaway 1--2]{ahn2022} propose a general mechanism for its occurrence.
We are interested in the first takeaway, which roughly states that under
suitable assumptions on the gradient map and on the critical points of the loss,
\emph{convergence to minima can only happen when initializing from a
set of measure zero}. It is not known if their assumptions hold in
practice; we prove these assumptions are true for linear networks.

In the case of linear neural networks with a quadratic loss function,
\textcite{bah2020} and \textcite{nguegnang2021} have performed a careful analysis of
both gradient flow and gradient descent and have found upper bounds on the
step-size that guarantee convergence of gradient descent without using the
Descent Lemma. However, as they already mention, the bounds are not tight, and
their experiments show that convergence also happens for step-sizes exceeding
their bounds. Besides the theoretical interest in finding a tight bound on the
step-size which ensures convergence, such a bound would show the limitations of
gradient descent and serve as a starting point for further analysis.

Some incentives to use large step-sizes, which have been empirically validated
\parencite{mohtashami2022,li2020}, are that gradient descent run with big step-sizes
converges faster and that solutions generalize better.

Motivated by the results mentioned above, in this paper we provide an answer to
the question posed above for the case of linear networks with a quadratic loss
function. More specifically, \textbf{we show that for all step-sizes bigger than
a constant $\eta_E > 0$, the set of initializations from which gradient descent
converges to a minimum has measure zero.} Additionally, in
Theorems~\ref{thm:analytic_non_quad_non-singular}
and~\ref{thm:measure_zero_non-linear} we provide general conditions on the loss
function which guarantee the same result, but which are easier to check for
in practice than those of \textcite[Assumption 1]{ahn2022}.

\subsection{Technical Contributions}
Motivated by the problem of understanding the dynamics of gradient descent with
a big step-size, the main contributions of this paper are as follows:
\begin{enumerate}
	\item We prove that for linear networks with quadratic loss, the minima are
		a smooth manifold (Proposition~\ref{thm:min_are_manifold}). This allows
		us to perform an analysis of the Hessian.

	\item Theorem~\ref{thm:gradient_descent_is_non-singular} proves that
		preimages of sets of measure zero under the gradient descent map $G:
		\theta \mapsto \theta - \eta \nabla L(\theta)$ are again sets of measure
		zero (such maps are called \emph{non-singular}) for losses associated to
		linear networks. 

	\item We study the dynamical stability of the fixed points under gradient
		descent iterates and prove that if the step-size is too large, then the
		set of initializations from which gradient descent converges to a
		critical point, \emph{including minima}, has measure zero
		(Theorem~\ref{thm:measure_of_convergence}). This means that for almost
		all initializations gradient descent will \emph{not} converge (see the
		example below and \figref{fig:2}). 

	\item We prove that, under suitable assumptions, the previous results hold
		for more general loss functions also
		(Theorems~\ref{thm:analytic_non_quad_non-singular} and
		\ref{thm:measure_zero_non-linear}). The result that the gradient descent
		map is non-singular is extended to non-linear networks with a bounded
		or analytic activation function; for other networks an easy to check
		criterion is given.
\end{enumerate}

\begin{example}
  \label{section:intuition}
  We illustrate the main results on a simple model
  which does \emph{not} have a Lispschitz continuous gradient: a two neuron linear
  network. With data $\left\{ 1 \mapsto 1 \right\}$, the loss and
  the eigenvalues of the Hessian matrix along the minima are:
  \begin{gather*}
	L(\theta_1,\theta_2)	= \frac{1}{2}(1-\theta_1\theta_2)^2, \\
	\lambda_{\text{max}}|_M(\theta_1,\theta_2) = \theta_1^2+\theta_2^2, \quad
	\lambda_{\text{min}}|_M(\theta_1,\theta_2) = 0.
	\label{eq:data_model}
  \end{gather*}

  The global minima are the hyperbolas $M = \left\{ \theta_1\theta_2 =1  \right\}$.
  $\lambda_{\text{max}}|_M$ is lower bounded by 2, thus $\eta_E = 1$
  (cf. Thm.~\ref{thm:measure_of_convergence}).
  For a fixed step-size $\eta > 0$, the set of weakly stable minima is $M_{WS} = \left\{
  \theta_1^2 + \theta_2^2 \leq 2/\eta \right\} \subset M$ (see also
  Definition~\ref{def:stable_manifold}). If $\eta<1$, then $M_{WS} \neq
  \varnothing$ and convergence is possible; for $\eta > 1$,
  convergence is impossible. See \figref{fig:example}.
\end{example}

\begin{figure}[ht]
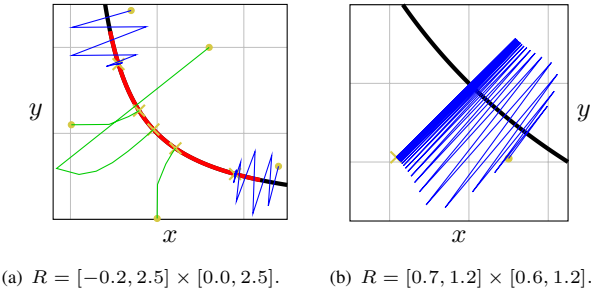

	\centering
	\subfigure[{$R = [-0.2,2.5]\times[0.0,2.5].$}]{
		\input{hyperbolae}
\tikzset{cross/.style={thick, cross out, draw=black, minimum size=2*(#1-\pgflinewidth), inner sep=0pt, outer sep=0pt},
cross/.default={1pt}}
\begin{tikzpicture}[scale=1.14]
	\draw[step=1.0cm,gray!50,ultra thin] (-0.2, 0.0) grid (2.5, 2.5);
	\draw (-0.2, 0.0) rectangle (2.5,2.5);
	\path 	node at (-0.2+2.7/2, -0.2) {$x$}
	node at (-0.4, 2.5/2) {$y$};	

	\poshyperbolae[ultra thick]{2.5}{2.5}
	\poshyperbolae[ultra thick,color=red]{2.189}{(1/0.457)}

        \filldraw [black!20!yellow] (2.4,0.617) circle (1pt);
        \draw (1.899,0.527) node[cross=3pt,black!20!yellow]{};
        \filldraw [black!20!yellow] (0.7,2.429) circle (1pt);
        \draw (0.554, 1.805) node[cross=3pt,black!20!yellow]{};
        \filldraw [black!20!yellow] (1.6, 2.0) circle (1pt);
        \draw (0.957, 1.045) node[cross=3pt,black!20!yellow]{};
        \filldraw [black!20!yellow] (1.0, 0.01) circle (1pt);
        \draw (1.211, 0.826) node[cross=3pt,black!20!yellow]{};
        \filldraw [black!20!yellow] (0.01, 1.1) circle (1pt);
        \draw (0.791, 1.264) node[cross=3pt,black!20!yellow]{};
        \draw[color=blue] plot file {tr_cv_eos.tex};
	\draw[color=green!80!black] plot file {tr_cv_mon.tex};
\end{tikzpicture}
		\label{fig:3}
	}
	\subfigure[{$R = [0.7,1.2]\times[0.6,1.2].$}]{
		\input{hyperbolae}
\tikzset{cross/.style={thick,cross out, draw=black, minimum size=2*(#1-\pgflinewidth), inner sep=0pt, outer sep=0pt},
cross/.default={1pt}}
\begin{tikzpicture}[scale=5.215]
	\draw[step=0.2cm,gray!50,ultra thin] (0.7, 0.65) grid (1.25, 1.2);
	\draw (0.7, 0.65) rectangle (1.25,1.2);
	\path 	node at ({0.7+(1.25-0.7)/2}, 0.65-0.04) {$x$}
	node at (1.25+0.04, {0.65+(1.2-0.65)/2}) {$y$};	

	\poshyperbolae[ultra thick]{1.25}{1.2}

        \filldraw [black!20!yellow] (1.1, 0.809) circle (0.25pt);
        \draw (0.817, 0.812) node[cross=3pt,black!20!yellow]{};
        \draw[color=blue,thin] plot file {tr_es_per.tex};
\end{tikzpicture}
		\label{fig:4}
	}
	\caption{
		Dynamics of gradient descent in two regimes for $L(x,y) =
		\frac{1}{2}(1-xy)^2$: (a) Trajectories for $\eta = 0.4$. $M_{WS}$ is red
		and iterations converge to points in $M_{WS}$; (b) Trajectory for $\eta = 1.1$. $M_{WS}
		= \varnothing$ and trajectories no longer converge. In both figures
		$R\subset\R^2$ is the displayed region.
	}
	\label{fig:example}
\end{figure}

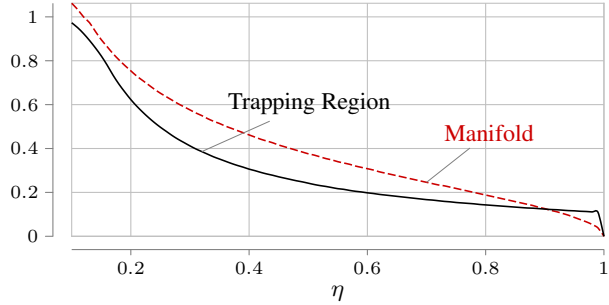
\begin{figure}[ht] 
    \centering
	\usetikzlibrary{datavisualization}

\begin{tikzpicture}[yscale=1,xscale=1.4]
	\datavisualization [
		scientific axes=clean, all axes={grid},
		x axis={label=$\eta$},
		visualize as smooth line/.list={trp, mfd}, 
		every data set label/.append style={text colored},
		trp={pin in data={text={Trapping Region}}},
		mfd={pin in data={text={Manifold}, when=x is 0.70}},
		style sheet = strong colors,
		style sheet = vary dashing]
		data [set=trp] {
			x,y
			0.10,0.973
			0.11,0.951
			0.12,0.924
			0.13,0.893
			0.14,0.859
			0.15,0.822
			0.16,0.780
			0.17,0.734
			0.18,0.693
			0.19,0.657
			0.20,0.623
			0.21,0.593
			0.22,0.566
			0.23,0.541
			0.24,0.518
			0.25,0.497
			0.26,0.478
			0.27,0.459
			0.28,0.442
			0.29,0.427
			0.30,0.412
			0.31,0.398
			0.32,0.386
			0.33,0.374
			0.34,0.362
			0.35,0.352
			0.36,0.342
			0.37,0.332
			0.38,0.323
			0.39,0.315
			0.40,0.306
			0.41,0.299
			0.42,0.291
			0.43,0.284
			0.44,0.278
			0.45,0.271
			0.46,0.265
			0.47,0.259
			0.48,0.253
			0.49,0.248
			0.50,0.242
			0.51,0.236
			0.52,0.231
			0.53,0.227
			0.54,0.222
			0.55,0.217
			0.56,0.214
			0.57,0.209
			0.58,0.205
			0.59,0.202
			0.60,0.198
			0.61,0.195
			0.62,0.191
			0.63,0.188
			0.64,0.185
			0.65,0.182
			0.66,0.179
			0.67,0.176
			0.68,0.173
			0.69,0.170
			0.70,0.167
			0.71,0.165
			0.72,0.162
			0.73,0.159
			0.74,0.157
			0.75,0.154
			0.76,0.152
			0.77,0.150
			0.78,0.148
			0.79,0.145
			0.80,0.143
			0.81,0.141
			0.82,0.139
			0.83,0.137
			0.84,0.135
			0.85,0.133
			0.86,0.131
			0.87,0.129
			0.88,0.128
			0.89,0.126
			0.90,0.124
			0.91,0.122
			0.92,0.121
			0.93,0.119
			0.94,0.117
			0.95,0.116
			0.96,0.114
			0.97,0.113
			0.98,0.111
			0.99,0.109
			1.00,0.000
		}
		data [set=mfd] {
			x,y
			0.1,1.062
			0.11,1.034
			0.12,1.0
			0.13,0.974
			0.14,0.934
			0.15,0.896
			0.16,0.864
			0.17,0.832
			0.18,0.804
			0.19,0.778
			0.2,0.754
			0.21,0.73
			0.22,0.71
			0.23,0.69
			0.24,0.67
			0.25,0.652
			0.26,0.636
			0.27,0.62
			0.28,0.604
			0.29,0.59
			0.3,0.576
			0.31,0.562
			0.32,0.55
			0.33,0.536
			0.34,0.524
			0.35,0.514
			0.36,0.502
			0.37,0.492
			0.38,0.482
			0.39,0.472
			0.4,0.462
			0.41,0.452
			0.42,0.442
			0.43,0.434
			0.44,0.426
			0.45,0.416
			0.46,0.408
			0.47,0.4
			0.48,0.392
			0.49,0.384
			0.5,0.376
			0.51,0.37
			0.52,0.362
			0.53,0.354
			0.54,0.348
			0.55,0.34
			0.56,0.334
			0.57,0.328
			0.58,0.32
			0.59,0.314
			0.6,0.308
			0.61,0.302
			0.62,0.294
			0.63,0.288
			0.64,0.282
			0.65,0.276
			0.66,0.27
			0.67,0.264
			0.68,0.258
			0.69,0.252
			0.7,0.246
			0.71,0.24
			0.72,0.234
			0.73,0.23
			0.74,0.224
			0.75,0.218
			0.76,0.212
			0.77,0.206
			0.78,0.2
			0.79,0.194
			0.8,0.188
			0.81,0.182
			0.82,0.176
			0.83,0.17
			0.84,0.164
			0.85,0.158
			0.86,0.152
			0.87,0.146
			0.88,0.14
			0.89,0.132
			0.9,0.126
			0.91,0.118
			0.92,0.112
			0.93,0.104
			0.94,0.096
			0.95,0.086
			0.96,0.076
			0.97,0.066
			0.98,0.054
			0.99,0.038
			1.0,0.0
		};
\end{tikzpicture}
	\caption{For the loss $\frac{1}{2}(1-xy)^2$, we plot the length of the manifold of weakly stable minima
	  $M_{WS}$ for $\eta \in [0.1, 1]$, normalised to the length
	  of $M_{WS}$ at $\eta=0.1$, which is $37.4$ (more details in the example above). As a
	  proxy for the size of $M_{WS}$ (since it can't be computed explicitly for
	  more complex losses), we also plot the ratio between the number of
	  initialisations that converge to a minimum (if this happens we say that
	  the initialisation lies in the \emph{trapping region} for that minimum) and the
	  total number of initialisations.
	  }
	\label{fig:2}
\end{figure}

\section{Related Work}
We present here a selection of the work relevant to our analysis. In the first part 
we present results about the landscape of linear networks and of wide
feed-forward networks with non-linear activation: understanding the landscape is
crucial for understanding the dynamics. In the second part 
existing results concerned with the general properties of optimization using
gradient descent on generic losses and more specific results for losses
determined by linear networks are presented. In the last part, we note some
works analyzing gradient descent as a dynamical system.

\subsection{Analyzing the Loss Landscape}

\textcite{baldi1989} have shown that for linear networks having only one hidden
layer (i.e. $\Phi(x) = W_2W_1x$)  all minima of $L$ are global minima. They
also conjectured that this result holds ``true in the case of linear networks
with several hidden layers'', which was proved only 20 years later in
\textcite{kawaguchi2016}. They further show that if a critical
point is not a minimum, then it must be a saddle point and also extend some of
their results to certain non-linear networks.

\textcite{trager2020} set up a geometric framework for analyzing the loss
landscape of neural networks. They look at the loss function as a composition of
two functions, $L = l \circ \mu$. The function $\mu$ is determined by the
architecture of the neural network and maps the space of parameters into the
space of functions that can be realized by the given architecture. $l$ is
usually a chosen convex functional, which also depends on the data. Because of
this representation of $L$, they distinguish between two types of the critical
points: \emph{pure} critical points which exist due to intrinsic the geometry and
non-convex nature of the functional space and the \emph{spurious} ones
otherwise. \textcite[Proposition 10]{trager2020} offers a new and
illuminating proof about the known results for the landscape of linear networks.
If $l$ is chosen to be the quadratic loss, particularly nice properties are
exhibited and these are also explained in their analysis.

A common hypothesis that is made when analyzing different optimization
algorithms along landscapes of neural networks is that the global minima $M$ of the
loss function determine a manifold \parencite[cf.]{li2022SGD, arora2022}.
\textcite{cooper2018} prove this assumption in the case of wide neural networks
with smooth rectified activation function. They show that $M$ is a smooth
manifold of dimension $(d_\theta-n)$, where $d_\theta$ and $n$ are the number of
parameters and data points, respectively. In practice, neural networks commonly
have many more parameters than data points, so their result strongly suggests
that in practical cases too $M$ is a very high-dimensional subset of
$\R^{d_\theta}$.

\subsection{Optimization using Gradient Descent}

\textcite{lee2019} study the convergence properties of gradient descent when the
step-size $\eta$ is kept constant and show, quite generally, that gradient
descent can converge to saddle points or maxima only when initialized from a set
of measure zero. This also answers a problem raised in \textcite{baldi1989}
regarding the ability of gradient descent and Newton's type methods to avoid
saddle points. \textcite{ahn2022}, like \textcite{lee2019} before them, use the Stable
Manifold Theorem to show that the same mechanism that explains why gradient
descent avoids maxima and saddles also makes gradient descent unable to converge
to minima which are not ``flat'' enough.  \textcite{chemnitz2024} provide a
non-trivial extension of the results of \textcite{ahn2022} to SGD. They introduce a
framework which allows one to differentiate between ``flat'' and ``curved''
minima in a stochastic setting as well. A new problem arises, namely that of finding and
characterizing the ``curved'' minima for losses of neural networks.
\textbf{We analyze the asymptotic behavior of the curvature along the minima for
linear networks with squared loss.}

\textcite{bah2020} take a different approach and study the dynamics of gradient
flow on the functional space determined by the mapping $\mu$. They show that
gradient flow always converges to critical points. \textcite{nguegnang2021}
follow up with an analysis of gradient descent: they provide relations on the
step-size during iterations that guarantee convergence. Unlike previous
results on the convergence of gradient descent, their bounds do not require
the step-size to decreases to zero; in fact it can even be kept constant. However, as they
observe in \textcite[Section 5]{nguegnang2021} their results are not entirely
tight: convergence can happen even if the step-size is bigger than their bounds.
Open questions still remain: \emph{Why can gradient descent converge with a
large \mbox{step-size}? What is the behavior of gradient descent iterates in
this regime?} Regarding the first question \textbf{we show that when the
step-size is chosen to be bigger than a critical value, gradient descent
iterates are unable to converge}.

\subsection{Gradient Descent and Dynamical Systems}
A common way in which the theory of dynamical systems is used to analyze
numerical methods, gradient descent in particular, is the following: considering
the continuous time limit of gradient descent, one obtains an ODE, usually
called the \emph{gradient flow}: $\dot{\theta} = -\nabla L(\theta)$. Next,
properties of this continuous time dynamical system and of its solution are
analyzed. The articles of \textcite{bloch1990} and \textcite{brockett1991} and the
references therein are good sources. Properties of the continuous time dynamical
system can then be used to derive new numerical schemes. The books
\textcite{helmke1994, hairer2010} are good introductions and the article
\textcite{rosca2023} contains a list of references which show how such methods have
been applied in the context of deep learning.

\section{Preliminaries}
In this section we set up the framework to present our result. We
start by formally defining the \emph{gradient descent map} for an arbitrary loss
function and some dynamical properties and quantities of interest (stability, limit
points); we also highlight the relationship between the local properties of the
loss and those of the dynamical system it determines. We end this section by
introducing the family of loss functions associated to a neural network. 

\subsection{Gradient Descent} \label{subsection:gradient_descent}
Given a $C^r, r\geq 1,$ loss function $L:\R^{d} \to \R$, the gradient
descent algorithm is used to find a (local) minimum of the loss. For a
fixed constant $\eta > 0$, and a point $\theta \in \R^{d}$, gradient
descent takes a step in the direction of steepest descent. This defines the
following $C^{r-1}$ discrete dynamical system on $\R^{d}$:
\begin{equation}
		G\colon \R^{d}  \to \R^{d}, \qquad 
		\theta  \mapsto \theta - \eta \nabla L(\theta).
	\label{eq:gd}
\end{equation}
\textcite{nguegnang2021} analyze the convergence of this dynamical system, and
provide a sufficient condition on the step-size $\eta$ given the initial point
that \emph{guarantees the convergence} of gradient descent (GD) iterates to a critical point of the
loss, however, as their own  experiments show, the bound is not tight, i.e.\
convergence can happen with larger step-sizes too. Our goal is to understand the
global dynamics of the GD map with an \emph{a priori} fixed learning rate
$\eta$, with an aim to understand the cases where it does \emph{not} converge.
We start by identifying the \emph{fixed points} (i.e. points $\overline{\theta}$ such that
$G(\overline{\theta})=\overline{\theta}$) of the GD map and determining whether they are
stable or unstable. To this end, we consider the notion of Lyapunov stability
\parencite{lyapunov}.
\begin{definition}[Lyapunov stability]
	For a continuous map $G\colon \R^{d} \to \R^{d}$, a fixed point $\overline{\theta}$
	is said to be (Lyapunov) \emph{stable} if, for any $\varepsilon > 0$, there exists
	a $\delta > 0$ such that, if $\norm{\theta-\overline{\theta}} < \delta$, then for
	every $n \in \mathbb{N}$, $\norm{G^n(\theta)-\overline{\theta}} < \varepsilon$. A
	fixed point that is not stable is called \emph{unstable}. Here, $G^n$ is $n$ times
	composition of $G$ with itself.  
	\label{def:stability}
\end{definition}
\begin{remark}[Stability of fixed points from linearization]
  \label{rem:stability_from_linearization}
  Throughout this paper by \emph{linearization} of a map $G$ at a point
  $\theta \in \R^d$ we mean its Jacobian matrix of partial derivatives:
  $DG(\theta) =
  [\partial G_i/\partial \theta_j]$.
  The linearization of $G$ at a fixed point $\overline{\theta} \in \R^d$ reveals
  information about the local behavior of the map. Intuitively, the eigenvalues of
  $DG(\overline{\theta})$ with moduli greater than one correspond to directions
  along which $G$ is expanding, the eigenvalues with moduli smaller than one
  correspond to directions along which $G$ is contracting, and the eigenvalues
  with moduli equal to one correspond to directions along which the norm remains
  constant. If all the eigenvalues of $DG(\overline{\theta})$ have moduli smaller
  than one, then $\overline{\theta}$ is stable; if there is at least one
  eigenvalue with modulus greater than one, then $\overline{\theta}$ is
  unstable; for a proof of this statement see
  \parencite[Theorem 15.16]{hale1991}. 
\end{remark}
\begin{definition}[Stable set]
  Let $G\colon \R^{d} \to \R^{d}$ be a continuous map and $\theta$ a point in
  $\R^{d}$. The \emph{stable set} of $\theta$, $W^s(G, \theta)$, consists of all
  points $\theta_0 \in \R^{d}$ such that they converge to $\theta$ forward in
  time, i.e.  there exists a sequence $(\theta_n) \subset \R^{d}$ with
  $\theta_{n+1}=G(\theta_{n})$ and $\lim_{n \to \infty} \theta_n = \theta$. When
  the map is clear from the context, we will suppress it from the notation,
  i.e.\ we write $ W^s(\theta) $ instead of $ W^s(G,\theta) $.
  \label{def:stable_set}
\end{definition}
In the context of the present work, a fixed point $\overline{\theta} \in \R^d$
of $G$ corresponds to a \emph{critical point} of $L$ (i.e. $\nabla
L(\overline{\theta})=0$), and the stable set of $\overline{\theta}$ corresponds
to the initializations which converge to $\overline{\theta}$ under iterates of
the GD map. We can also define the stable set for a set: if $U \subset \R^d$,
then $W^s(G,U) \coloneqq \cup_{\theta \in U} W^s(G,\theta)$.
\begin{definition}[Notation: $M, Crit(L), H_L$]
We let $M$ denote the set of global minima of the loss $L$ and $Crit(L)$ denote
the set of critical points of $L$. For any point $\theta \in
\R^{d}$, we denote the Hessian of the loss at that point as $H_L(\theta)$. 
\end{definition}
Note that $M \subseteq Crit(L)$. As an example, if $L(x,y) = (1-xy)^2$, then
$Crit(L) = \left\{ 0 \right\} \cup \left\{(x,y)\in \R^2|xy = 1\right\}$ and $M =
\{(x,y)\in\R^2|xy = 1\}$. We are interested in understanding how the
``size'' of $W^s(G, M)$ changes as we change $\eta$ for the loss functions
we define in the next section.
\begin{remark}[Hessian determines linearization $DG(\theta)$]
  \label{rem:eigv_and_hess}
  From the definition of $G$ we see that the linearization at any point $\theta
  \in M$ is $DG(\theta) = I - \eta H_L(\theta)$. This relation, coupled with
  Remark~\ref{rem:stability_from_linearization}, justifies the
  importance of understanding the Hessian matrix at points in $M$. 
\end{remark}

\subsection{Loss Landscape of Neural Networks} \label{subsection:loss_landscape}
A fully connected feed-forward neural network with $h$ hidden layers and
activation function $\sigma\colon \R \to \R$ is a map $\Phi_\sigma\colon \R^{d_\theta}
\times \R^{d_0} \to \R^{d_h}$ of the form
  \begin{equation}
	  \Phi_{\sigma}(\theta, x) = \sigma(W_h \dots \sigma(W_1x+b_1) + \cdots + b_h), 
	  \label{eq:non-linear_nn}
  \end{equation}
where $\theta = (W_h,b_h,\dots,W_1,b_1) \in \R^{d_\theta}$ are the parameters,
$W_i \in \R^{d_i\times d_{i-1}}$ are matrices, and
$d_\theta= \sum_{i=0}^{h-1}(d_i+1)d_{i+1}$ is the total number of parameters; the
activation function $\sigma$ is applied element-wise. Given a network and a
dataset $\{(x_i \mapsto y_i)\}_{i=1}^n$ of $n$ points and labels, the goal of
supervised learning is to find a set of parameters $\theta^*$ such that
the output of the network defined by $\theta^*$ evaluated on any given
point will be close to the labels (even on points not in the dataset). To this
end, another function, call it $l$, is used as a measure of closeness between
the output of the network and the labels. In the next section, we will choose $\sigma$
to be the identity, set $b_i =0 $, and use the quadratic distance for $l$ (unless explicitly
mentioned otherwise). We call such networks \emph{linear} networks. This loss
can be written explicitly as
\begin{equation}
	L(\theta) = \sum_{i=1}^n \norm{y_i - \Phi(\theta, x_i)}^2 = 
	\sum_{i=1}^n \norm{y_i - W_h\dots W_1x_i}^2.
	\label{eq:loss_explicit}
\end{equation}
We write the loss as the composition $L = l \circ \mu_d$, where
\begin{equation}
	\begin{aligned}
		\mu_d \colon \R^{d_h\times d_{h-1}}\times\dots\times\R^{d_1\times d_0} &
		\rightarrow \R^{d_h\times d_0} \\
		(W_h, \dots, W_1) & \mapsto  W_h\dots W_1
	\end{aligned}
	\label{eq:loss}
\end{equation}
is the \emph{matrix multiplication map} and $l$ is a functional on the
space of $(d_h\times d_0)$-matrices (in our case the quadratic distance). The explicit form for the quadratic loss
with fixed data and label matrices, $X \in \R^{d_0 \times n}$ and $Y \in \R^{d_h
\times n}$, where $n$ is the number of data points, is
\begin{equation}
		l\colon \R^{d_h\times d_0} 	 \rightarrow \R, \qquad
		W							 \mapsto	 \norm{Y-WX}^2.
	\label{eq:ecld_loss}
\end{equation}
Under the assumption that the matrix $XX^T$ has full rank, there is a global
minimum $W^*_0 = YX^T(XX^T)^{-1/2}$, and it has $\rank
W^*_0 = \min\{d_h, d_0\}$. Throughout this paper we assume that $ XX^T $ has full
rank. To find the global minima for the whole
loss, $L$, we have to look at the preimage of $W^*_0$ under the matrix
multiplication map $\mu_d$. It can happen that $W^*_0$ does not lie in the
image of $\mu_d$. This is the case when one of the hidden layers of the network
is narrower than the input or the output dimension. Hence, \textcite{trager2020}
categorize linear network architectures into two types---filling and
non-filling.
\begin{definition}[Filling and non-filling architectures]
	Let $r = \min \{d_h, \dots, d_0\}$. If $r = \min \{d_h, d_0\}$, we say
	that the map $\mu_d$ is \emph{filling} since the image of $\mu_d$ is the whole
	space $\R^{d_h\times d_0}$. Otherwise, we say that $\mu_d$ is  
	\emph{non-filling}. We define $\mathcal{M}_r$ as the space of $(d_h
	\times d_0)$-matrices of rank at most $r$, which coincides with the image of $\mu_d.$
	\label{def:filling_non-filling}
\end{definition}
In the non-filling case $\mathcal{M}_r$ is no longer convex, however \textcite[Theorem
12]{trager2020} shows that the $l$ restricted to $\mathcal{M}_r$ still has a
\emph{unique} global minimum.
\begin{proposition}[\textcite{trager2020}]
	If the singular values of $W^*_0$ are pairwise distinct and
	positive, $l|_{\mathcal{M}_r}$ has a unique global minimum. In the filling
	case, the minimum of $l|_{\mathcal{M}_r}$ agrees with that of $ l $. In the
	non-filling case, the minima are different, however, if $ W^* $ is the
	minimum of $l|_{\mathcal{M}_r}$, we have that $ \rank W^* = \min\{d_h, \ldots, d_0\} $.
	\label{prop:minima_on_det_var}
\end{proposition}
\begin{remark}[Our analysis is based on the framework from \textcite{trager2020}]
	The set $\mathcal{M}_r$ is known as a \emph{determinantal variety} and has been thoroughly
	studied in algebraic geometry; its dimension is
	$r(d_h + d_0 - r)$. \textcite{trager2020} introduce this
	framework for analyzing general losses and prove results relating to the
	geometry of the losses determined by linear networks.
	\label{remark:trager}
\end{remark}

\section{Main Results: Linear Networks}
We begin this section by investigating the geometry of $M$ and the Hessian of
$L$ at points in $M$. We proceed to explain the asymptotic behaviour of the
eigenvalues of the Hessian along $M$. Using the relation between the Hessian of
the loss and the Jacobian of the gradient descent map, we show that the gradient
descent map is non-singular, i.e. its inverse image preserves null-sets. In the
last part we present a stability analysis of the minima together with results
about the convergence of gradient descent for large step-sizes.

\subsection{Geometry of the minima and the spectrum of $H_L$} \label{subsection:minima_are_manifold}
We first present the main result about the global minima of $L$. While the
results are intuitively simple, we are unaware of such a characterization in the
neural network literature.
\begin{proposition}[Geometry of $M$]
  \label{thm:min_are_manifold}
  The set $M$ of global minima is a $(d_\theta - r(d_0+d_h-r))$-dimensional smooth manifold.
\end{proposition}
Before proving the proposition, we need another result from \textcite{trager2020}.
\begin{proposition}[{\textcite[Theorem 4]{trager2020}}]
	Let $r = \min \{d_h, \ldots, d_0\}$, $ W $ a matrix in the image of $
	\mu_d $,  $W = \mu_d(\theta)$, and $ \theta = (W_1, \ldots, W_h) $.
	\begin{itemize}
		\item (Filling case) If $r = \min \{d_h, d_0\}$, the differential
			$D\mu_d(\theta)$ has maximal rank equal to $\dim \mathcal{M}_r = d_0d_h$ if 
			and only if, for every $i \in \{1, 2, \dots, h-1\}$, either
			$\rank(W_{>i}) = d_o$ or $\rank(W_{<i+1}) = d_i$ holds.
		\item (Non-filling case) If $ r < \min \{d_h, d_0\}$, the differential $ D\mu_d(\theta)$ 
			has maximal rank equal to $ \dim \mathcal{M}_r = r(d_0+d_h-r)$  if
			and only if $ \rank(W^*)=r.$
	\end{itemize}
	\label{prop:matrix_mul_map}
\end{proposition}
\begin{proof}{\emph{of Proposition~\ref{thm:min_are_manifold}.}}
	Proposition \ref{prop:minima_on_det_var} says that in both cases, filling
	and non-filling, $l$ has a unique minimum. $M$ is the preimage of this
	minimum under the matrix multiplication map $\mu_d$.  In both the filling
	and non-filling case, Proposition~\ref{prop:matrix_mul_map} states that $
	D\mu(\theta) $ has maximal rank for all $ \theta \in M $. Hence, $\mu_d$ is
	a submersion at each $\theta \in M$, meaning that $M$ is a regular level set
	of $\mu_d$.  We apply the Regular Level Set Theorem \parencite[Corollary
	5.14]{lee2012} to finish the proof.
\end{proof}

\textcite{kawaguchi2016} has investigated the structure of the set of minima for
linear networks under the quadratic loss. They show that all minima are global
minima. However, they do not prove that the minima define a manifold, and we
are not aware of any similar result in the case of linear networks. The fact
that $M$ is a manifold provides a nice way to describe the spectrum of
$H_L(\theta)$ for $\theta \in M$. 

\begin{proposition}[Geometry of $L$ near $M$]
  \label{thm:curvature_near_manifold}
  The Hessian of $L$ at a point $\theta \in M$ has ($d_\theta - \dim M$) positive
  eigenvalues, $\dim M$ zero eigenvalues, and no negative eigenvalues.
\end{proposition}
\begin{proof}
	At every point $\theta \in \R^{d_\theta}$, $H_L(\theta)$ is a real symmetric matrix,
	thus has a basis of real eigenvectors. If $\theta \in M$, $H_L(\theta)$ has
	non-negative eigenvalues since $\theta$ is a minimum. The kernel of
	$H_L(\theta)$ is at least $\dim M$ dimensional since $M$ is a manifold. To
	see this, take any $\dim M$-slice chart $\varphi$ around $\theta$. In this
	chart, $L$ can be written as $$\hat{L}(\hat{\theta}) =
	L\circ\varphi^{-1}(\hat{\theta}_1,\cdots,\hat{\theta}_{\dim \mathcal{M}_r},0,\cdots,0),$$
	where $\hat{\theta} = (\hat{\theta}_1,\cdots,\hat{\theta}_{\dim
	\mathcal{M}_r},0,\cdots,0)$ is the coordinate representation of $\theta$ in the chart
	$\varphi$. Since $L$ is constant along $\dim M$ directions in coordinates, the
	dimension of the kernel of the Hessian of $L$ in coordinates is at least
	$\dim M$. We have that $$\hat{H_L} = D{\varphi^{-1}}^T H_L D\varphi^{-1}.$$ Since
	$\varphi$ is a diffeomorphism, its derivative is invertible, and it follows that
	the kernel of $H_L$ has dimension at least $\dim M$.

	At any point $ \theta \in \R^{d_\theta} $ the Hessian is:
	\begin{equation}
		H_L(\theta) = D\mu_d(\theta)^T D^2l(\mu_d(\theta)) D\mu_d(\theta) + Dl(\mu_d(\theta))
		D^2\mu_d(\theta).
		\label{eq:hessian}
	\end{equation}

	The gradient of $L$ vanishes at $\theta$, thus the second term in
	\eqref{eq:hessian} vanishes. $D\mu_d(\theta)$ and
	$D^2l(\mu_d(\theta))$ both have rank $\dim \mathcal{M}_r$ thus $H_L(\theta)$
	has rank at most $\dim \mathcal{M}_r$. It follows that the kernel of $H_L(\theta)$ has
	rank at most $d_\theta - \dim \mathcal{M}_r = \dim M$. The rank of the kernel of
	$H_L(\theta)$ is then $\dim M$ which is equal to the number of zero
	eigenvalues. This, together with the fact that $H_L(\theta)$ is positive
	semi-definite, shows there are $\dim \mathcal{M}_r$ positive eigenvalues.
\end{proof}
The previous theorem, coupled with the fact from Remark~\ref{rem:eigv_and_hess}
that the eigenvalues of the linearization of $G$ at a point $\theta$
are determined by the eigenvalues of the Hessian $H_L(\theta)$,
allows us to characterize the asymptotic behavior of points near $M$. A more
comprehensive treatment of stability is presented in the last part of this
section.
\begin{remark}[Geometry of $M$ for non-linear networks]
	\textcite{cooper2018} has investigated the structure of the global minimizers
	for overparameterized networks in the interpolating regime that use smooth rectified activation
	functions. They show that the minimizers form a sub-manifold of
	$\R^{d_\theta}$, where $d_\theta$ is the number of parameters, and that the
	Hessian of the loss at these points has only non-negative eigenvalues. We
	cannot directly use their results since feed-forward networks with identity
	activation function cannot operate in the interpolating regime,
	thus \textcite[Lemma 3.3]{cooper2018} does not hold; however, we derived analogous results for linear networks.
	\label{rem:cooper_manifolds}
\end{remark}

\subsection{Asymptotic behaviour of the Hessian on $M$}
\label{subsection:curvature_on_M}

The Hessian can be interpreted as describing the local curvature at minima and
the geometric picture one should have in mind is that the loss function becomes
increasingly curved, the curvature eventually tending to infinity, as one gets
further away from the origin of Euclidean space. The next theorem makes this
intuition precise.
\begin{theorem}[The eigenvalues of the Hessian are proper]
	Let $\lambda_{i}:M \rightarrow \R$ denote the function on $M$ that maps any
	$\theta \in M$ to $\lambda_i(H_L(\theta))$, the $i$-th non-zero eigenvalue
	of the Hessian of the loss function. Then $\lambda_{i}$ is a proper
	function, i.e. if $K \subset \R$ is any compact set, then
	$\lambda_i^{-1}(K)$ is also a compact set. In particular, for any $k>0$, the
	set $\{\theta \in M | \lambda_i(H_L(\theta)) \in [0,k]\}$ is compact. 
	\label{thm:eigenvalues_are_proper}
\end{theorem}
The proof of this statement is mostly computational, so we have deffered it to
the appendix.
In other words, Theorem \ref{thm:eigenvalues_are_proper} states that the points
$\theta \in M$ where the $i$-th eigenvalue of the Hessian is bounded by a
positive constant $k>0$ form a bounded set in $M$ and thus in $\R^{d_\theta}$.
An immediate consequence is that the directional curvature of $L$ at all other points in $M$
must be bigger than $k$. The following corollary relies on this observation to
show that if the step-size $\eta$ is larger than a threshold value, then the
loss around every minimum has a high curvature along some direction.
\begin{corollary}[For large $\eta$, the loss has high curvature]
	For each $i$-th eigenvalue of $H_L$ not equal to zero, there exists a
	step-size $\eta_i > 0$ such that for all step-sizes $\eta > \eta_i$,
	$ \lambda_i $ is lower bounded by $ 2/\eta $, i.e.\ $\lambda_i^{-1}([0,2/\eta]) = \varnothing$.
	\label{cor:diff_stability}
\end{corollary}
\begin{proof}
	Take any $\theta \in M$. Since $\lambda_i$ is a proper map,
	the preimage of $I=[0,\lambda_i(\theta)]$, $\lambda_i^{-1}(I) \neq
	\varnothing$, is a compact
	set. A continuous function attains its minimum on a compact set, denote it
	by $\lambda_{i, \text{min}}$. This is, in fact, the global minimum of $\lambda_i$
	on $M$ since for any $\theta' \in M \setminus
	\lambda_i^{-1}(I)$ we have $\lambda_i(\theta') >
	\lambda_i(\theta)$. Let $\eta_i \coloneqq 1/\lambda_{i, \text{min}}$; then, for
	any $\eta > \eta_i$ we have that $\lambda_i^{-1}([0, 2/\eta]) \subset
	\lambda_i^{-1}([0, 2/\eta_i)) = \varnothing$.
\end{proof}
Using Corollary~\ref{cor:diff_stability} together
Remarks~\ref{rem:stability_from_linearization} and \ref{rem:eigv_and_hess}, and
Proposition~\ref{thm:curvature_near_manifold} we can describe the local behavior of
$G$ as follows: From Corollary~\ref{cor:diff_stability} we know there exists an
$\eta_E > 0$ such that for any $\eta > \eta_E$ and $\theta \in M$, any
non-zero eigenvalue of $H_L(\theta)$ is greater than $2/\eta$.
From Remark~\ref{rem:eigv_and_hess} we get that the eigenvalues of $DG(\theta)$
are either $1$ or have moduli greater than one. This means that all minima are
unstable. This is important since we will show that $W^s(M)$ has measure
zero in this case, i.e. convergence to any minima happens only from
initializations that lie in a set of measure zero.

\subsection{Gradient descent map is non-singular}
\label{subsection:gd_non-singular}

To prove the statement in the previous paragraph, we need a technical result
about the gradient descent map $G$: it preserves sets of measure zero under preimages. This
result is also interesting in itself since it is needed to prove statements in
the case of non-linear networks as well. In the following, by a null-set or a
measure-zero set, we mean a set having measure zero w.r.t.\ the Lebesgue measure
on $ \R^n $; we also say that a property holds almost everywhere if it does not
hold only on a measure zero set.
\begin{definition}[Non-singular maps]
	Let $f\colon \R^n \rightarrow \R^n$ be a smooth map. If $f^{-1}(B)$ is a
	null-set for any null-set $B$, we say that $f$ is a \emph{non-singular} map.
	\label{def:non-singular_map}
\end{definition}
We now prove a few lemmas which are used to show that all smooth maps 
from $\R^n$ to $\R^n$ that are invertible almost everywhere are non-singular.
\begin{lemma}[Images of measure zero sets have measure zero]
	Suppose $A \subset \R^n$ has measure zero and $f\colon \R^n \rightarrow \R^n$
	is a smooth map. Then $f(A)$ has measure zero.
	\label{lem:smooth_map_measure_zero}
\end{lemma}
\begin{proof}
	See \textcite[Proposition 6.5]{lee2012}.
\end{proof}
\begin{lemma}[Local diffeomorphisms are non-singular]
	Let $U \subset \R^n$ be an open set and $f \colon U \rightarrow \R^n$ be a
	smooth map such that for all $x \in U$, $\det(DF(x)) \neq 0$. Then
	$f^{-1}(B)$ has measure zero for any measure zero set $B \subset \R^n$.
	\label{lem:local_diffeomorphism_non-singular}
\end{lemma}
\begin{proof}
	The condition that $\det(DF(x)) \neq 0$ for any $x \in U$ together with the
	inverse function theorem \parencite[Theorem 4.5]{lee2012} tells us that $f$ is 
	a local diffeomorphism.

	Let $B \subset \R^n$ be any measure zero set. Because $f$ is a local
	diffeomorphism, around every point $x \in f^{-1}(B)$ there exists an open
	set $U_x \ni x$ such that $f$ restricted to it, $f|_{U_x}\colon U_x
	\rightarrow f(U_x)$, is a diffeomorphism. Since $f^{-1}(B) \subset \R^n$ and
	$\R^n$ is second countable, we can extract a countable subcover $\left\{
	U_{x_i} \right\}_{i=1}^\infty$ of $f^{-1}(B)$. Because $B$ has measure zero,
	so does each intersection $B_i \coloneqq B \cap f(U_{x_i})$. Since $f$
	restricted to $U_{x_i}$ is a diffeomorpishm, by Lemma 
	\ref{lem:smooth_map_measure_zero}, $f|_{U_i}^{-1}(B_i)$ has measure zero
	for all $i$. Observe that 
	$$f^{-1}(B) \subset \cup_{i=1}^\infty f|_{U_i}^{-1}(B_i).$$
	A countable union of measure zero sets has measure zero and a subset of a
	measure zero set has measure zero, thus the lemma is proved.
\end{proof}
\begin{proposition}[Maps invertible a.e. are non-singular]
	Let $f\colon \R^n \rightarrow \R^n$ be a smooth map such that the set of points
	where the determinant of its Jacobian matrix vanishes, $S(f)$, has measure
	zero. Then $f$ is non-singular.
	\label{prop:invertible_ae_is_non-singular}
\end{proposition}
\begin{proof}
	The determinant is a continuous map and a singleton is closed, thus
	$S(f)$ is closed. Then, the restriction of $f$ to the open set $U \coloneqq 
	\R^n\setminus S(f)$ is a local diffeomorphism.

	Let $B \subset \R^n$ be any measure zero set. Its preimage under $f$ can be
	written as
	$$f^{-1}(B) = f|_{S(f)}^{-1}(B) \cup f|_{U}^{-1}(B).$$
	The first term in the union is the intersection of a set with $S(f)$, which
	has measure zero by assumption, thus it has measure zero. The second term
	has measure zero by Lemma 
	\ref{lem:local_diffeomorphism_non-singular}. It follows that $f^{-1}(B)$
	has measure zero. 
\end{proof}
\begin{remark}[Proposition \ref{prop:invertible_ae_is_non-singular} can be generalized considerably] 
  \label{rem:extending}
  A smooth map between smooth manifolds $f\colon M \rightarrow
  N$ of dimensions $m$ an $n$, respectively, is non-singular if the set of
  critical points of $f$ has measure zero.
\end{remark}
Having found a simpler condition to check that a map is non-singular,
we show that the gradient descent map for linear neural networks is
invertible almost everywhere.
\begin{lemma}[The zero set of polynomials has measure zero]
	\label{lem:measure_zero_set_poly}
	Let $p\colon \R^n \rightarrow \R$ be a non-zero polynomial function.
	Denote by $Z(p)\coloneqq \left\{ x \in \R^n | p(x) = 0 \right\}$ the zero
	set of $p$. Then, $Z(p)$ has measure zero.
\end{lemma}
\begin{proof}
	We proceed by induction on $n$. If $n = 1$, using the Fundamental Theorem of
	Algebra, there are finitely many roots for the polynomial $p$. Thus $Z(p)$
	has measure zero.

	Suppose the lemma is established for polynomials in $n-1$ variables.
	Let $p$ be a nontrivial polynomial in $n$ variables, say of degree
	$k \geq 1$ in $x_n$. We can then write

	\begin{equation*}
		p(x, x_n) = \sum_{j=0}^k p_j(x)x_n^j,
	\end{equation*}

	where $x = (x_1, \cdots, x_{n-1})$ and $p_0, \cdots, p_k$ are polynomials in
	$n-1$ variables, where at least $p_k$ is non-trivial. Now if
	$(x, x_n)$ is such that $p(x, x_n) = 0$ there are two possibilities:

	\begin{enumerate}
		\item $p_0(x) = \cdots = p_k(x) = 0$, or
		\item $x_n$ is a root of the (non-trivial) one-variable polynomial 
			$p_x(t) = \sum_{j=0}^k p_j(x)t^j$.
	\end{enumerate}

	Let $A, B$ be the subsets of $\R^n$ where these respective conditions hold,
	so that $Z(p) = A \cup B$. By the induction hypothesis $A$ has measure zero.
	By the Fundamental Theorem of Algebra, for each fixed $x$, there are
	finitely many $t$ such that $(x,t) \in B$. Using Fubini's Theorem
	\parencite[Theorem 3.4.4]{bogachev2007}, we
	conclude that $B$ has measure zero.
\end{proof}
\begin{lemma}[Gradient descent map for linear networks is invertible almose everywhere]
	\label{lem:gd_sing_measure_zero}
	Let $L$ be the loss function given by a linear network with quadratic loss
	function and $G$ the associated gradient descent map. For any $\eta > 0$ the
	set of points where the Jacobian matrix of $G$ is singular, $S(G) \coloneqq
	\left\{ \theta \in \R^{d_\theta} | \det(DG(\theta)) = 0 \right\}$, has measure
	zero.
\end{lemma}
\begin{proof}
	The map $p: \theta \mapsto \det (I -\eta H_L(\theta))$ is a polynomial
	(since the entries of $H_L$ are polynomials), but it is not obvious why the
	entries in the matrix do not cancel out to the zero polynomial when taking
	the determinant. After diagonalizing $H_L$, for any $\theta \in M$,
	$p(\theta) = \Pi_{i=1}^{\dim M}(1-\eta\lambda_i(\theta))$. We can always
	find $\theta$ such that $p(\theta)\neq0$ using the fact that the non-zero
	eigenvalues are proper maps (Theorem \ref{thm:eigenvalues_are_proper}).
\end{proof}
\begin{theorem}
	The gradient descent map $G$ is non-singular.
	\label{thm:gradient_descent_is_non-singular}
\end{theorem}
\begin{proof}
	The proposition follows from Proposition
	\ref{prop:invertible_ae_is_non-singular} and Lemma
	\ref{lem:gd_sing_measure_zero}.
\end{proof}
The results of Lemma \ref{lem:measure_zero_set_poly} can be extented to hold not
only for polynomial functions but, in fact, for analytic functions. We will need
such a result later.
\begin{lemma}[The zero set of analytic functions has measure zero]
	\label{lem:measure_zero_set_analytic}
	Let $f\colon \R^n \rightarrow \R$ be a non-zero real analytic function. Then the zero
	set of $f$, $Z(f)$, has measure zero.
\end{lemma}
\begin{proof}
	The proof proceeds by induction on $n$ and is similar to the proof of Lemma
	\ref{lem:measure_zero_set_poly}. Let $f \colon \R \rightarrow \R$ be a real
	analytic function. Suppose that $Z(f)$ is not a set having measure zero. Since it is not a
	null-set, it is an uncountable set, thus it has an accumulation
	point. From the identity theorem \parencite[Corollary 1.2.7]{krantz2002} we
	get that $f \equiv 0$, contradicting our assumption that $f$ is non-zero.

	Suppose the lemma is established for real analytic functions defined on
	$\R^{n-1}$. Assume that $Z(f)$ has is not a null-set. By Fubini's Theorem
	there is a a set $E \subset \R$ not having measure zero such that for all $x
	\in E$ the set $Z(f)\cap (\{x\}\times \R^{n-1})$ does not have measure zero. By the
	induction hypothesis we conclude that $f = 0$ on each of those
	hyperplanes. Since $E$ is uncountable, it has an accumulation point
	$a \in \R$. Thus there is a sequence of distinct points $a_k \mapsto a$ such
	that $f = 0$ on the hyperplanes with first coordinate in $\left\{ a_1, a_2, \cdots
	\right\} \cup \left\{ a \right\}$.

	Now take any line of the form $\R \times \{x_2\} \times \cdots \times \{x_n\}$,
	where $x_i \in \R$ are fixed numbers. Since $f$ is real analytic on this
	line and $f = 0$ on a set with an accumulation point on that line, we have
	$f = 0$ on that line. This was for true any line, hence $f \equiv 0$ on
	$\R^n$.
\end{proof}

\subsection{Dynamics of gradient descent} \label{subsection:dynamics}

We now arrive at the main results of this paper, where we characterize the
unstable fixed points of gradient descent with a pre-specified step-size $\eta>0$.
For this, we need to define a subset of $M$ which determines the quantitative
behaviour of the dynamics of gradient descent.
\begin{definition}[Weakly stable minima]
	Let $\eta > 0 $. We say a minimum $\theta \in M$ is \emph{weakly stable} if
	the smallest non-zero eigenvalue of the Hessian at $\theta$ is smaller or
	equal to $2/\eta$. We denote the set of all weakly stable minima by
	$M_{WS}$. 
	\label{def:stable_manifold}
\end{definition}
\begin{theorem}[On stability of minima]
  \label{thm:stability_of_minima}
  Let $\theta$ be a minimum. If $\theta \in M\setminus M_{WS}$ then it
  is unstable. If $\theta$ is stable, then $\theta \in M_{WS}$.
\end{theorem}
\begin{proof}
  If $\theta \in M\setminus M_{WS}$, then there is at least one eigenvalue of
  $DG(\theta)$ whose modulus is greater than one. Hence, there is at least one direction along
  which $G$ is expanding, thus $\theta$ is unstable (cf.
  Remark~\ref{rem:stability_from_linearization}).

  If $\theta \in M$ is stable, then the moduli of the eigenvalues of $DG(\theta)$ are
  smaller or equal to one (cf. Remark~\ref{rem:stability_from_linearization}). Coupled
  with Remark~\ref{rem:eigv_and_hess}, we get that $|\lambda_i(\theta)| \leq
  2/\eta$ for all $i$.
\end{proof}
Theorem \ref{thm:stability_of_minima} provides us
with information only about the local behavior of $G$ for points near $M_{WS}$,
but more precise results can also be derived using additional properties of the
gradient descent map. In Definition~\ref{def:stable_set} and the paragraph
after it, we defined the stable set for a point $\theta$ (and a set) as the
totality of points which converge to $\theta$ under iterates of $G$. This means
that $W^s(M)$ consists of all points which converge to a minimum. 
The following theorem lets us determine the qualitative behavior of the dynamics
of gradient descent by examining only $M_{WS}$.
\begin{theorem}[Measure of $W^s(M)$]
  \label{thm:measure_of_convergence}
  There exist an $\eta_E > 0$ such that for any $\eta > \eta_E$, the set
  $W^s(M)$ of initializations that converge to a minimum has measure zero.
\end{theorem}
We state \textcite[Theorem 1]{ahn2022}, which we will use in proving
Theorem~\ref{thm:measure_of_convergence}.  We note that their
proof is an application of the Stable Manifold Theorem \parencite[Theorem III.7]{shub1987}.
\begin{lemma}[{\textcite[Theorem 1]{ahn2022}}]
	\label{lem:ahn}
	Suppose that for each minimum $\theta \in M$, it holds that
	$\lambda_{\max}(\theta) > 2/\eta$. Further assume that $G$ is non-singular
	and that for each minimum $\theta \in M$, the Jacobian matrix of $G$ is
	invertible. Then there is a measure-zero subset $N$ such that for all
	initializations $\theta \in \R^{d_\theta} \setminus N$ the gradient descent
	dynamics $\theta_{i+1} = G(\theta_i)$ does not converge to any of the minima
	in $M$.
\end{lemma}
\begin{proof}[of Theorem~\ref{thm:measure_of_convergence}]
	Lemma~\ref{lem:gd_sing_measure_zero} states that the map $ G $ is
	non-singular. Since $ \eta > \eta_E $, we have that the absolute value of
	the eigenvalues of $ DG(\theta) $ is greater or equal to $ 1 $ for all $
	\theta \in M $, hence $ DG(\theta) $ is invertible. The assumptions of
	Lemma~\ref{lem:ahn} are satisfied, thus we can apply it.
\end{proof}
The previous theorem tells us that if the step-size gets too large, then
convergence is possible only when initializing from a set having measure zero.
\begin{remark}[Statements similar to Theorem~\ref{thm:measure_of_convergence} exist]
  Statements similar to Theorem \ref{thm:measure_of_convergence} about the
  instability of critical points, especially saddles, have also been made in
  \textcite{ahn2022} and in \textcite{lee2019}.
  The latter show that the necessary condition for gradient
  descent to avoid maxima and saddle points is for $G$ to be non-singular. They
  show this is true for a large class of loss functions under two assumptions:
  Lipschitz continuity of the gradient and a step-size smaller than the
  sharpness. As we saw in the introduction, these are unreasonable assumptions
  for models used in practice. \textbf{We showed that these assumptions are not
  necessary for linear networks with quadratic loss function.}

  \textcite[Theorem 1]{ahn2022} prove a generic theorem about convergence of
  gradient descent to critical points, and we even use it in our proof of
  Theorem~\ref{thm:measure_of_convergence}. However, their proof is based on
  \textcite[Assumption 1]{ahn2022}, an assumption stated without trying to verify
  if it is relevant for models used in practice. \textbf{We showed that
  \textcite[Assumption 1]{ahn2022} actually holds for linear networks with
  quadratic loss.}
  \label{rem:what_us_new}
\end{remark}

\section{Extension: Non-linear Networks}
We now state generalizations of
Theorems~\ref{thm:gradient_descent_is_non-singular} and
\ref{thm:measure_of_convergence} to more general loss functions, under
reasonably weak conditions. 

\begin{theorem}
  If $L\colon \R^d \to \R$ is an analytic function and one eigenvalue of
  its Hessian is not constant then, for any $\eta > 0$, the gradient descent 
  map $G\colon \theta \mapsto \theta - \eta\nabla L(\theta)$ is
  non-singular.
  \label{thm:analytic_non_quad_non-singular}
\end{theorem}
\begin{proof}
  By assumption, there exist some eigenvalue $ \lambda $ of $ H_L $ and two
  points $ \theta_1, \theta_2 \in \R^d $ such that $ \lambda(\theta_1) \neq
  \lambda(\theta_2) $.  Take any analytic path $\gamma\colon \R \to \R^d$ such
  there exists $t_1\neq t_2\in \R$ with $\gamma(t_1) = \theta_1$ and
  $\gamma(t_2) = \theta_2$. This can be done since we can interpolate any finite
  number of points using analytic functions. Since $\gamma$ and the map $\theta
  \mapsto H_L(\theta)$ are both analytic, their composition, $t \mapsto
  H_L\circ\gamma(t) \in \R^{d\times d}$, is analytic as well. In this case, the
  eigenvalues of $H_L\circ\gamma$ can be globally parameterized by analytic
  functions, i.e. there exist analytic functions $\lambda_i\circ\gamma\colon \R
  \to \R$ which are equal to the eigenvalues of $H_L\circ\gamma(t)$ for all
  $t\in\R$ \parencite{kato1995}.

  For a non-constant analytic function, the set of where it is equal to a
  constant $c\in \R$, $Z(\lambda_i\circ\gamma - c) \subset \R$
  has measure zero (Lemma~\ref{lem:measure_zero_set_analytic}). For any $\overline{t} \in \R \setminus\cup_{i=1}^d
  Z(\lambda_i\circ\gamma) \subset \R$, it holds that
  $\lambda_i\circ\gamma(\overline{t}) \neq c$. In particular, for any $\eta>0$ there exists
  $\overline{\theta} \coloneqq \gamma(\overline{t}) \in \R^d$ such that $\det(I
  - \eta H_L(\overline{\theta})) \neq 0$. We have showed that the analytic map
  $ \theta \mapsto \det(I-\eta H_L(\theta)) $ is not constant, thus by
  Lemma~\ref{lem:measure_zero_set_analytic} its set of zeros has measure zero.
  We conclude that $G$ is invertible almost everywhere and an application of
  Proposition~\ref{prop:invertible_ae_is_non-singular} finishes the proof.
\end{proof}
We now give an easy to check criterion to determine if a given loss satisfies
the assumptions of Theorem~\ref{thm:analytic_non_quad_non-singular}.
\begin{proposition}
  Let $L$ be a bounded analytic function. Then, the eigenvalues of the Hessian of
  $L$ cannot be constant.
\end{proposition}
\begin{proof}
  We prove the statement by contradiction. Suppose that $L$ is a bounded,
  analytic function and that at least one of the eigenvalues of the Hessian, say
  $\lambda_l$, is constant.

  Consider the orthonormal eigenvalue basis $(v_1, \dots, v_d)$ given by the
  eigenvalue decomposition of the Hessian at $0$. Since $L$ is analytic, it is
  equal to its Taylor series expansion, which we can write w.r.t. the basis
  above as follows:
  \begin{equation}
	\label{eq:taylor_series}
	L(\theta) = L(0) + \sum_{n=1}^\infty \frac{1}{n!} \sum_{i_1,\dots,i_n=1}^d
	\nabla_{v_{i_1}}\dots\nabla_{v_{i_n}}L (0) \theta_{i_1}\dots \theta_{i_n}.
  \end{equation}

  From the assumption that at least one of the eigenvalues of $H_L$ is constant, 
  we get the following expression for the higher order derivatives of $L$
  \begin{equation*}
	\nabla_{v_{i_n}}\dots\nabla_{v_{i_2}}\nabla_{v_{i_1}}L(0) = \left\{
	  \begin{array}{cl}
		0, &\text{ if } i_1 \neq i_2,\\
		\nabla_{v_{i_n}}\dots\nabla_{v_{i_3}}\lambda_l(0) = 0, &\text{ if } i_1 = i_2 = l, \\
		\nabla_{v_{i_n}}\dots\nabla_{v_{i_3}}\lambda_{i_1}(0), &\text{ if } i_1 = i_2 \neq l.
	  \end{array}
	\right.
  \end{equation*}

  We see that for a choice of $\theta_l = t \in \R$ and $\theta_{i\neq l} = 0$, all terms
  of order higher than $2$ vanish in equation~\eqref{eq:taylor_series}:
  \begin{equation}
	L(0, \dots, 0, t, 0, \dots, 0) = L(0) + \nabla_{v_{l}}L(0)t +
	\lambda_l(0)t^2.
	\label{eq:loss_is_quad}
  \end{equation}

  This shows that $L(t) \to \infty$ as $t \to \infty$. Thus, $L$ cannot be
  bounded.
\end{proof}
\begin{example}
  \label{rem:const_l_quad_L} One way to think about
  Theorem~\ref{thm:analytic_non_quad_non-singular} is that the loss function
  ``is not quadratic''. More precisely, if one of the eigenvalues, say
  $\lambda$, were constant, then we can integrate along a curve
  $\gamma_\lambda\colon \R\to\R^d$ to get that $L\circ\gamma_\lambda(t) =
  at^2+bt+c$. Since this is true for any $t\in \R$ we have that
  Theorem~\ref{thm:analytic_non_quad_non-singular} is true for any bounded loss
  function. In particular, it holds true for any non-linear neural network where
  the activation function is \emph{analytic} and \emph{bounded}, e.g. sigmoid or
  hyperbolic tangent, and $l$ is \emph{analytic} on the image of $\mu$, e.g.
  MSE. 
\end{example}

Under the assumption that the eigenvalues of the Hessian are bounded from
below by a positive constant we find that convergence to critical points is
impossible (except for initializations lying in a set of measure zero) even for
non-linear networks.
\begin{theorem}
  Let $L\colon \R^d \to \R$ be as in
  Theorem~\ref{thm:analytic_non_quad_non-singular} and $M$ its set of global
  minima. Additionally, if there exists a non-zero eigenvalue $\lambda_p$ for every $\theta
  \in M$ and $\inf_{\theta\in M}\lambda_p(\theta) > 0 $, then there exists
  $\eta_E > 0$ such that for all $\eta > \eta_E$, the set $W^s(M)$ has
  measure zero.
  \label{thm:measure_zero_non-linear}
\end{theorem}
\begin{proof}
  Let $\lambda_E = \inf_{\theta \in M}(\lambda_p(\theta)) > 0$. We get that $\eta_E =
  2/\lambda_E$ and that for all $\eta > \eta_E$, $M_{WS} = \varnothing$.
  From here on, the proof is identical to as that of Theorem~\ref{thm:measure_of_convergence}.
\end{proof}

\subsection{Experiments}
In this section we validate our results. We hypothesize that
Theorem~\ref{thm:measure_zero_non-linear} also holds for analytic loss functions
which are not bounded. We consider the following network architecture: fully
connected with $784$ neurons in the input layer, $32$ in the hidden layers, and
$2$ in the output layer. The number of hidden layers is specified when presenting
the result. The networks in \figref{fig:conv_fig} use GELU activation function
and squared loss; those in \figref{fig:exp} use ReLU activation and softmax with
cross-entropy as loss. In both cases, for each different step size we train on a
subset MNIST \parencite{lecun1998} consisting of images of only zeros and ones. For
the networks in \figref{fig:conv_fig} a batch consisting of all training
samples is used; in \figref{fig:exp} we use minibatches of size $64$.
\pgfplotstableread[row sep=\\, col sep=&]{
  step 	& 4 & 8 & 16 & 32\\
  0.125	& 1 & 1 & 1  & 1 \\
  0.25  & 1 & 1 & 1 & 1\\
  0.5 	& 1 & 1 & 1 & 1\\
  1	  	& 0 & 1 & 1 & 1\\
  2	  	& 0 & 0 & 0 & 0\\
  4	  	& 0 & 0 & 0 & 0\\
}\mydata
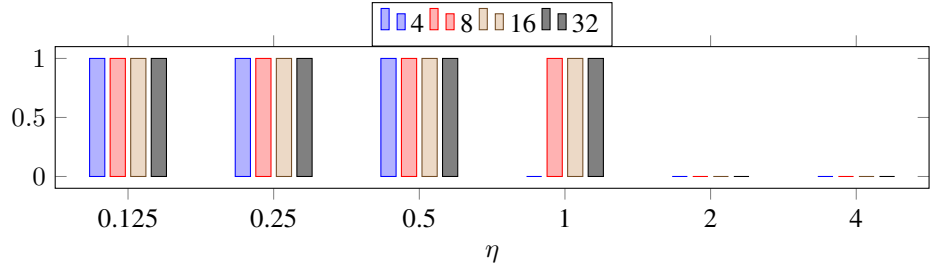
\begin{figure}[h]
  \centering
  \begin{tikzpicture}[]
	\begin{axis}[
		ybar,
		bar width=0.2cm,
		width=\linewidth,
		height=3.46cm,
		xlabel={{$\eta$}},
		symbolic x coords={0.125,0.25,0.5,1,2,4},
		xtick=data,
		legend style={at={(0.5, 1.0)},anchor=south, legend columns=-1},
	  ]
	  \addplot table[x=step,y=4]{\mydata};
	  \addplot table[x=step,y=8]{\mydata};
	  \addplot table[x=step,y=16]{\mydata};
	  \addplot table[x=step,y=32]{\mydata};
	  \legend{4,8,16,32}
	\end{axis}
  \end{tikzpicture}
  \caption{Percentage of trajectories that converge after gradient descent
	training with fixed step size on MNIST.  GELU and MSE were used as activation
	function and loss. The number in the legend specifies the network depth.
	Convergence does not happen for $\eta > 1$.}
  \label{fig:conv_fig}
\end{figure}
Even for non-analytic activation function and SGD, we expect that, as the step
size gets bigger, fewer trajectories will converge, leading to a behavior for
the size of the trapping region similar to the one in Figure~\ref{fig:2}. Figure~\ref{fig:exp} presents our findings. For details see the appendix.
\begin{figure}[h] 
    \centering
	\usetikzlibrary{datavisualization}

\begin{tikzpicture}[yscale=1,xscale=1]
	\datavisualization [
		scientific axes=clean, all axes={grid},
		x axis={label=$\eta$, function=\pgfdvmathexp{\pgfvalue}{\pgfvalue}},
		visualize as smooth line/.list={mod2, mod4, mod8, mod12}, 
		every data set label/.append style={text colored},
		mod2={label in legend={text={2}}},
		mod4={label in legend={text={4}}},
		mod8={label in legend={text={8}}},
		mod12={label in legend={text={12}}},
		style sheet = strong colors,
		style sheet = vary dashing]
		data [set=mod2] {
			x,y
			0.0078125,1.0
			0.009290680585958758,1.0
			0.011048543456039806,1.0
			0.013139006488339289,1.0
			0.015625,1.0
			0.018581361171917516,1.0
			0.02209708691207961,1.0
			0.026278012976678578,1.0
			0.03125,1.0
			0.03716272234383503,1.0
			0.04419417382415922,1.0
			0.052556025953357156,0.94
			0.0625,0.81
			0.07432544468767006,0.72
			0.08838834764831845,0.61
			0.10511205190671431,0.34
			0.125,0.18
			0.14865088937534013,0.08
			0.1767766952966369,0.05
			0.21022410381342863,0.04
			0.25,0.02
		}
		data [set=mod4] {
			x,y
			0.0078125,1.0
			0.009290680585958758,1.0
			0.011048543456039806,1.0
			0.013139006488339289,1.0
			0.015625,1.0
			0.018581361171917516,0.99
			0.02209708691207961,0.99
			0.026278012976678578,0.99
			0.03125,0.97
			0.03716272234383503,0.9
			0.04419417382415922,0.83
			0.052556025953357156,0.65
			0.0625,0.52
			0.07432544468767006,0.41
			0.08838834764831845,0.23
			0.10511205190671431,0.13
			0.125,0.06
			0.14865088937534013,0.05
			0.1767766952966369,0.04
			0.21022410381342863,0.02
			0.25,0.01
		}
		data [set=mod8] {
			x,y
			0.0078125,1.0
			0.009290680585958758,1.0
			0.011048543456039806,0.99
			0.013139006488339289,0.89
			0.015625,0.85
			0.018581361171917516,0.79
			0.02209708691207961,0.77
			0.026278012976678578,0.69
			0.03125,0.58
			0.03716272234383503,0.51
			0.04419417382415922,0.36
			0.052556025953357156,0.31
			0.0625,0.16
			0.07432544468767006,0.13
			0.08838834764831845,0.05
			0.10511205190671431,0.04
			0.125,0.02
			0.14865088937534013,0.01
			0.1767766952966369,0.02
			0.21022410381342863,0.01
			0.25,0.0
		}
		data [set=mod12] {
			x,y
			0.0078125,1.0
			0.009290680585958758,1.0
			0.011048543456039806,0.99
			0.013139006488339289,0.98
			0.015625,0.95
			0.018581361171917516,0.83
			0.02209708691207961,0.71
			0.026278012976678578,0.57
			0.03125,0.36
			0.03716272234383503,0.19
			0.04419417382415922,0.14
			0.052556025953357156,0.1
			0.0625,0.02
			0.07432544468767006,0.03
			0.08838834764831845,0.0
			0.10511205190671431,0.01
			0.125,0.01
			0.14865088937534013,0.0
			0.1767766952966369,0.0
			0.21022410381342863,0.0
			0.25,0.0
		};
\end{tikzpicture}
	\caption{Size of the trapping region
	  (as a percentage of trajectories that converge out of 100 total
	  initialisations) for a range of step sizes and models trained with SGD on
	  MNIST. The number in the legend represents the number of hidden layers. }
	\label{fig:exp}
\end{figure}
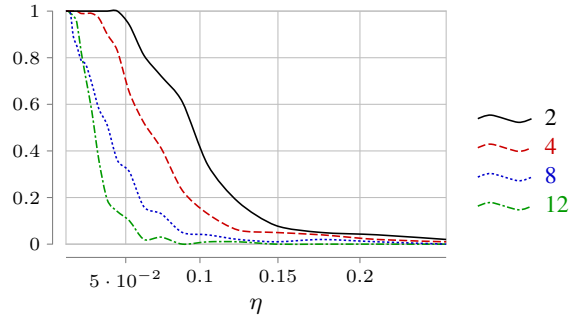

\section{Discussion}
This work has analyzed the convergence properties of gradient descent in the
large learning rate regime. In the case of linear neural networks, we prove that
convergence is impossible, except for initializations from a set of measure zero
(Theorem~\ref{thm:measure_of_convergence}). In
Theorem~\ref{thm:measure_zero_non-linear} the same result is proven, under mild
assumptions, for more general loss functions. The validity of these assumptions
is shown through experiments.  We conclude by discussing some open questions.

\paragraph{Do Theorems~\ref{thm:analytic_non_quad_non-singular}
and~\ref{thm:measure_zero_non-linear} hold for non-linear networks?} If the
activation function and $l$ are analytic (e.g. sigmoid or the hyperbolic
tangent), then the loss and $G$ are analytic. If $L$ is bounded, we showed that
the eigenvalues of $H_L$ can't be constant. Activation functions such as ReLU
need not have this property, and it is not clear if $G$ is non-singular in this
case. The additional assumption of Theorem~\ref{thm:measure_zero_non-linear}
(namely that there exists a non-zero eigenvalue of the Hessian that is lower
bounded by a positive constant along $ M $) is harder to verify as it
necessitates a comprehensive study of the minima of the loss function. As a
counterexample, the loss function $(x,y) \mapsto \frac{1}{2}(y^2-x^2)^2$ has
vanishing Hessian at $(0,0)$.

\paragraph{Are the minima manifolds for non-linear networks too?}
It has been shown that for non-linear networks operating in the interpolating
regime the minima form a manifold \parencite{cooper2018}. However, this is not a
necessary condition for the minima to be a manifold, since linear networks
cannot be in the interpolating regime for nonlinear data, yet we showed that
their global minima form a manifold. We are not aware of any straightforward
methods to show that the minima determine a manifold for non-linear networks
that are not interpolating, neither of any methods to easily adapt our results
for linear networks even to networks with simple activation functions such as
polynomial functions.

\paragraph{What happens when the loss functional $l$ is changed?}
To show that the minima of the loss form a manifold, we used the fact that
the loss functional $l$ has a unique minimum in both the filling and
non-filling case. This is a special property of the quadratic loss and
\textcite{trager2020} show that if one applies even infinitesimal
changes to it, the situation changes drastically: it is no longer the
case that $l|_{M_r}$ admits a unique minimum for all choices of
$W^*_0$. There are two open directions regarding this aspect. First, how
does the loss landscape change for different loss functionals $l$ in the
case of linear networks? Second, what are the critical
points of $l$ restricted to the functional space determined by
non-linear networks? 

\paragraph{Can gradient descent avoid saddles and maxima?}
\textcite{lee2019} use the assumptions of Lipschitz continuity of the gradient
and the fact that the step-size is bounded by the inverse of the sharpness to
show that the gradient descent map $G$ is a local diffeomorphism and thus
conclude that gradient descent can't converge to saddles or maxima (except for
initialisations from a set of measure zero). These assumptions are needed
twofold: first to be able to apply the Stable Manifold Theorem referenced
from~\textcite{shub1987}; second to show that $ G $ is non-singular. We have seen
that the gradient descent map is non-sigular and invertible almost everywhere
even if the assumptions are dropped. More general versions of the Stable Manifold
Theorem have been proved in~\textcite{ruelle1980} and~\textcite{ruelle1979}. An
application of the results in these papers might show that the ability of
gradient descent to avoid critical points and saddles is generic, in the sense
that is hold for almost all analytic loss functions.

\paragraph{Existence of forward-invariant subsets?}
The more difficult question becomes the second assumption in
\textcite{ahn2022}: do there exist any compact forward-invariant sets? An answer to
this question would provide a more complete understanding of optimization in the
edge of stability regime that gradient descent with a big step-size has been observed to
operate in: our results explain why, for a big step-size, gradient descent cannot
converge; an answer to the question related to the existence of
forward-invariant sets would explain why gradient descent does not diverge to
infinity. After this qualitative analysis, a more quantitative one could be
pursued: how do the attractors look like and what kind of asymptotic dynamics
takes place inside them?

\begin{appendices}

\section{Proof that $\lambda_{i}$ are proper maps}
\newcommand{\N}{\mathbb{N}}
\newcommand{\abs}[1]{|#1|}

To prove Corollary \ref{cor:diff_stability} we need to show that the
eigenvalues of the Hessian of the loss function have a global minimum. For
linear networks it actually turns out that they are proper maps (which, in
finite-dimensional vector spaces, is the same as saying that they are
norm-coercive, i.e. the norm of the function goes to infinity as the norm of the
argument does). We start by going over some topological facts. It is necessary to have
this topological discussion since a description of the asymptotic behaviour of
functions defined on a manifold through the analysis of extensions of these
functions on the Euclidean space in which the manifold is embedded requires a
careful investigation. A first example where a naive application of the definitions
gives unexpected results comes from considering the open ball of unit radius
embedded in Euclidean space. The sequence $\left\{ \frac{i-1}{i} e
\right\}_{i=1}^\infty$ diverges to infinity as a sequence in the open ball for any
unit vector $ e $, however it is a convergent sequence in the ambient Euclidean space. After
the topological discussion, we show  that the eigenvalues are proper maps,
proving Theorem \ref{thm:eigenvalues_are_proper}.

\subsection{Some topological facts}
\begin{definition}[\bf Proper maps]
	A continuous map $f\colon X \to Y$ between two topological spaces $X$ and
	$Y$ is proper if the preimage of any compact set is compact.
	\label{def:proper_map}
\end{definition}

We use the additional structure of the problem to find an equivalent condition.
In second countable Hausdorff topological spaces, we can check that a map is
proper using sequences.

\begin{definition}
	A sequence $(x_i)$ in a topological space $X$ is said to diverge to infinity
	if for every compact set $K \subseteq X$ there are at most finitely many
	values $i$ for which $x_i \in K$.
	\label{def:div_to_infty}
\end{definition}

\begin{proposition}
	If $X$ is a second countable Hausdorff space and $f\colon X \to Y$ a continuous
	map, then $f$ is proper if and only if $f$ takes sequences diverging to
	infinity in $X$ to sequences diverging to infinity in $Y$.
	\label{prop:proper_seq}
\end{proposition}

\begin{proof}
	We apply \textcite[Proposition 4.92]{lee2011} and \textcite[Proposition
	4.93]{lee2011}.
\end{proof}

The setting of our problem consists of a manifold $M$ sitting in an ambient Euclidean space.
To use Proposition \ref{prop:proper_seq} we have to take sequences in $M$
diverging to infinity. The next propositions show that a sequence in $M$
diverges to infinity if and only if there exists a sequence in $\R^{d_\theta}$
that diverges to infinity and additionally all terms of the sequence lie in $M$.
This seems like a trivial statement, but it does not necessarily hold if the
ambient space is not Hausdorff or $M$ is not closed (see the example in the
short discussion at the beginning).

\begin{proposition}
	Let $X$ be a Hausdorff space and $X' \subset X$ a closed subspace. Then, a
	set $K' \subset X'$ is compact if and only if there exists a compact subset
	$K \subset X$ such that $K' = K \cap X'$.
	\label{prop:compact_subspaces}
\end{proposition}

\begin{proof}
	To prove one direction, let $K' \subset X'$ be a subset such that there
	exists a compact subset $K \subset X$ with $K' = K \cap X'$. Since
	$K$ is compact and $X$ is Hausdorff, $K$ is closed. $K'$ is also closed
	since it is the intersection of two closed sets. But it is also a subset of
	a compact space ($K' \subset K$) thus it is itself compact.

	For the other direction note that if $K'$ is compact (and compactness is an
	intrinsic property of a space), it does not matter if we view it as a subset
	of $K'$ or of $K$.
\end{proof}

\begin{proposition}
	Let $X$ be a Hausdorff space and $X' \subset X$ a closed subspace. A
	sequence of points $(x_i)$ in $X'$ diverges to infinity if and only if it
	diverges to infinity as a sequence in $X$.
	\label{prop:div_seq_subpaces}
\end{proposition}

\begin{proof}
	Assume that $(x_i)$ is a sequence in $X'$ diverging to infinity and let
	$K \subset X$ be any compact set. The set $K' = K \cap X$ is compact by the
	previous proposition and because $(x_i)$ is a sequence diverging to
	infinity, only a finite number of terms lie inside $K'$. The others lie in
	$X' \setminus K'$ which does not intersect $K$, thus we have that only a finite
	number of terms of the sequence can lie in $K$. 

	Conversely, assume that $(x_i)$ is a sequence in $X$ diverging to infinity
	with the property that it is also a sequence in $X'$. Let $K' \subset X'$ be
	any compact set. By Proposition \ref{prop:compact_subspaces} there exists a compact set
	$K \subset X$ such that $K' = K \cap X'$. But by assumption only finitely
	many terms of the sequence lie in $K$, hence in $K'$.
\end{proof}

Now we use the fact that $\R^{d_\theta}$ is a finite-dimensional normed space to
arrive at the intuitive notion of a sequence diverging to infinity.

\begin{proposition}
	In a finite-dimensional normed vector space $X$, a sequence $(x_i)$ diverges
	to infinity if and only if $\norm{x_i} \to \infty$.
	\label{prop:di_in_nvs}
\end{proposition}

\begin{proof}
	Any finite dimensional normed vector space $X$ is homeomorphic to $\R^d$
	equipped with the usual topology, where $d = \dim X$, thus we will write
	$\R^d$ instead of $X$.

	Suppose $(x_i)$ diverges to infinity. This means that for any compact set
	$K\subseteq \R^d$ there are only finitely many values for $i$ for which $x_i
	\in K$. Let $r$ any positive number. The closed ball of radius $r$ around
	$0$, $\overline{B_r(0)}$, is compact, thus, there exists $i_r$ such that $\forall j
	\geq i_r$, $x_j \notin \overline{B_r(0)}$, i.e. $\norm{x_j} > r$.

	Conversely, suppose that $(x_i)$ is a sequence in $\R^n$ such that
	$\norm{x_i} \to \infty$. Let $K \subseteq \R^n$ be any compact set (i.e.
	closed and bounded for $\R^d$). There exists an $r_K > 0$ such that $K
	\subseteq B_{r_K}(0)$. By our hypothesis on the sequence we have that there
	exists a number $i_K \in \N$ such that $\forall j > i_K$, $\norm{x_j} > r_K$,
	i.e. $x_j \notin K$. Thus the sequence $(x_i)$ has only finitely many terms
	in $K$.
\end{proof}

\begin{proposition}
	Let $M$ be a properly embedded submanifold of $\R^n$. If $M$ is 
	a closed subset of $ \R^n $, a continuous function $f\colon M \rightarrow \R$ is proper
	if and only if for any sequence $\norm{\theta_i} \rightarrow \infty$ with
	$\left\{  \theta_i\right\}_{i \in \N} \subset M$, the sequence
	$(f(\theta_i))$ diverges to infinity.
	\label{prop:prop_in_manifolds}
\end{proposition}

\begin{proof}
	Let $ i:M \to \R^n $ be the embedding of $ M $ in $ \R^n $.
	We note that a set $K \subset M$ is compact if and only if there is a
	compact set $\tilde{K} \subset i(M)$ such that $\tilde{K} = i(K)$ (since $M$
	is homeomorphic to its image under $i$).

	Applying Propositions
	\ref{prop:div_seq_subpaces} and \ref{prop:di_in_nvs} finishes the proof.
\end{proof}

\subsection{$\lambda_i$ is a proper map}

To prove that each non-zero eigenvalue function $\lambda_i$ is proper, it suffices to
show that the smallest non-zero eigenvalue $\lambda_{\text{min}}$ is proper. Since
$\lambda_i \geq \lambda_{\text{min}}$ we have that if $\lambda_{\text{min}}$ diverges to infinity,
then $\lambda_i$ also diverges to infinity for all $i$ such that $\lambda_i$ is
not zero.

We can write $\lambda_{\text{min}}$ as

\begin{equation}
	\lambda_{\text{min}}(\theta) = \inf_{e \in S^{d_\theta},\\ e \notin \ker
	D\mu_d(\theta)} \norm{(D\mu_d(\theta)e)^T H (D\mu_d(\theta)e)}.
	\label{eq:min_eigenvalue}
\end{equation}

Observe that $\lambda_{\text{min}}$ is a well-defined function for all $\theta \in
\R^{d_\theta}$ not just for points in $M$. Thus our work would be greatly
simplified if we could show that $\lambda_{\text{min}}$ diverges to infinity for any
sequence of points in $\R^{d_\theta}$ diverging to infinity not just for those
who are also in $M$. This is not the case, since we can always find directions
in $\R^{d_\theta}$ such that $\lambda_{\text{min}}$ is zero along them. However, we can use
a ``trick'' to avoid these ``bad'' directions. Recall that $M$ was defined by $d_0d_h$ equations

\begin{equation}
	F_i(\theta) \coloneqq \mu_{d,i}(\theta) - c_i = 0,
	\label{eq:eq_def_M}
\end{equation}

where $\mu_{d,i}(\theta)$ and $c_i$ are the $i$'th entries in the matrices
$\mu(\theta)$ and $W^*$, respectively. Since these functions are by
definition zero on $M$, for any function $F\colon \R^{d_0d_h} \rightarrow \R$
which maps zero to zero, we can add $F(F_1(\theta), \cdots, F_{d_0d_h}(\theta))$
to $\lambda_{\text{min}}$ without changing the value of $\lambda_{\text{min}}$ on $M$. Define
$\tilde{\lambda}$ by 

\begin{equation}
	\tilde{\lambda}(\theta) \coloneqq \lambda_{\text{min}}(\theta) +
	\norm{\theta}^{2h}\sum_{i=1}^{d_0d_h} F_i(\theta)^2.
	\label{eq:lambda_modified}
\end{equation}

This function is well defined on $\R^{d_\theta}$. We show
that $\tilde{\lambda}$ tends to infinity as the norm of its argument tends to
infinity. We have the following equalities

\begin{align}
	\begin{split}
		\tilde{\lambda}\left( \theta \right)& =\tilde{\lambda}\left(\norm{\theta}\frac{\theta}{\norm{\theta}}\right),\label{eq:prop_lam_tilde_1}
	\end{split}
	\\
	\begin{split}
		& =\lambda\left(\norm{\theta}\frac{\theta}{\norm{\theta}}\right) + \norm{\theta}^{2h} \sum_{i=1}^{d_0d_h} F_i\left(\theta\right)^2,\label{eq:prop_lam_tilde_2}
	\end{split}
	\\
	\begin{split}
		& = \norm{\theta}^{2(h-1)}\lambda\left(\frac{\theta}{\norm{\theta}}\right)
		+
		\norm{\theta}^{2(h-1)}\sum_{i=1}^{d_0d_h}F_i\left(\frac{\theta}{\norm{\theta}}\right)
		-
		\norm{\theta}^{2(h-1)}\sum_{i=1}^{d_0d_h}F_i\left(\frac{\theta}{\norm{\theta}}\right)\\
		&   + \norm{\theta}^{2h} \sum_{i=1}^{d_0d_h}
		F_i(\theta)^2,\label{eq:prop_lam_tilde_3}
	\end{split}
	\\
	\begin{split}
		& =\norm{\theta}^{2(h-1)}\tilde{\lambda}(\frac{\theta}{\norm{\theta}}) +
		\sum_{i=1}^{d_0d_h} \left[
			\left(\norm{\theta}^hF_i(\theta)\right)^2 -
			\left(\norm{\theta}^{(h-1)}F_i(\frac{\theta}{\norm{\theta}})\right)^2 
		\right],\label{eq:prop_lam_tilde_4}
	\end{split}
	\\
	\begin{split}
		&= \norm{\theta}^{2(h-1)}\tilde{\lambda}(\frac{\theta}{\norm{\theta}})
		+
		\sum_{i=1}^{d_0d_h}R_i(\theta).\label{eq:prop_lam_tilde_5}
	\end{split}
\end{align}

For \eqref{eq:prop_lam_tilde_2} we use the definition of $\tilde{\lambda}$;
\eqref{eq:prop_lam_tilde_3} uses the fact that $D\mu_d(c\theta) =
c^{h-1}D\mu_d(\theta)$ and we add and subtract the sums where the $F_i$'s have
unit vectors, $\frac{\theta}{\norm{\theta}}$, as argument; for
\eqref{eq:prop_lam_tilde_4} we use the definition of
$\tilde{\lambda}$ again; and for \eqref{eq:prop_lam_tilde_5} we define the
functions $R_i$:

\begin{equation}
	R_i\left( \theta \right) \coloneqq \left(\norm{\theta}^hF_i(\theta)\right)^2
	- \left(\norm{\theta}^{(h-1)}F_i(\frac{\theta}{\norm{\theta}})\right)^2. 
	\label{eq:def_ri}
\end{equation}

We now analyze the asymptotic behaviour of the two terms in
\eqref{eq:prop_lam_tilde_5}.

\begin{enumerate}
	\item $\norm{\theta}^{2(h-1)}\tilde{\lambda}(\frac{\theta}{\norm{\theta}})$.

		Since the sphere $S^{d_\theta}$ is compact and $\tilde{\lambda}$ is a
		continuous function, it attains its minimal value on the sphere,
		$\tilde{\lambda}_m$. This value is strictly greater than zero. To see
		this, consider two cases: if $\frac{\theta}{\norm{\theta}} \in M$, then
		$\lambda\left(\frac{\theta}{\norm{\theta}}\right)$ is positive since
		$D\mu\left(\frac{\theta}{\norm{\theta}}\right)$ has full rank; if
		$\frac{\theta}{\norm{\theta}} \notin M$, then at least one of the
		equations defining $M$, say the $j$'th, is non-zero, $F_j\left( \frac{\theta}{\norm{\theta}}
		\right) \neq 0$, thus $\tilde{\lambda}\left(
		\frac{\theta}{\norm{\theta}} \right)$ is again strictly positive.

		We can conclude that 
		$$ \norm{\theta}^{2(h-1)}\tilde{\lambda}(\frac{\theta}{\norm{\theta}})
		\geq \norm{\theta}^{2(h-1)}\tilde{\lambda}_{\text{min}} \rightarrow \infty \text{ as
		} \norm{\theta} \rightarrow \infty.$$

	\item $R_i(\theta)$.

		Assume that $\mu_{d,i}\left( \frac{\theta}{\norm{\theta}} \right) \neq 0$ or
		$c_i \neq 0$. If both are zero, then $R_i(\theta) = 0$, thus
		$R_i(\theta)$ is lower bounded. We know that at
		least one $c_i \neq 0$ since otherwise $W^*$ would have rank
		zero. We now show that $R_i(\theta)$ is lower bounded. Writing it out
		explicitly we have
		\begin{equation*}
			\begin{split}
				R_i(\theta) =& \left( \norm{\theta}^{4h} - \norm{\theta}^{2(h-1)} \right)\mu_{d,i}\left(
				\frac{\theta}{\norm{\theta}} \right)^2 - 2\left( \norm{\theta}^{3h} -
				\norm{\theta}^{2(h-1)} \right) \mu_{d,i}\left( \frac{\theta}{\norm{\theta}}
				\right)c_i \\
						&+ \left(  \norm{\theta}^{2h} - \norm{\theta}^{2(h-1)}\right) c_i^2.
			\end{split}
		\end{equation*}
	
		If $\mu_{d,i}\left( \frac{\theta}{\norm{\theta}} \right) = 0$, then
		$R_i(\theta) = \left(  \norm{\theta}^{2h} -
		\norm{\theta}^{2(h-1)}\right) c_i^2$ which is lower bounded. If
		$\mu_{d,i}\left( \frac{\theta}{\norm{\theta}} \right) \neq 0$ then the term
		in degree $4h$ dominates and $R_i(\theta)$ is lower bounded. If
		remains to show that $R_i(\theta)$ is lower bounded for any $\theta \in
		\R^{d_\theta}$. This follows since $\mu$ is a continuous function and
		the sphere is compact, thus $\mu$ attains its maximum and minimum on the
		sphere and we have that $\mu_{d,i}(e)^2$ takes all values in the interval
		$\left[ 0, \max(\mu_{d,i,m}^2, \mu_{d,i,M}^2) \right]$ where $\mu_{d,i,m}$ and
		$\mu_{d,i,M}$ are the minimum and the maximum values of $\mu_{d,i}$ on the
		sphere, respectively. Since the interval is compact, we get that
		$R_i(\theta)$ is lower bounded for any $\theta \in \R^{d_\theta}$.
\end{enumerate}

We see that one of the terms defining $\tilde{\lambda}$ is lower bounded,
whereas the other one diverges to infinity, hence their sum also diverges to
infinity. This proves the following theorem.

\begin{theorem}[Non-zero eigenvalues of $H_L$ are proper maps]
	Each non-zero $\lambda_i\colon M \rightarrow \R$ is a proper map. 
	\label{thm:lambda_prop_map}
\end{theorem}

\end{appendices}

\printbibliography

@article{ruelle1979, title={Ergodic theory of differentiable dynamical systems}, volume={50}, ISSN={1618-1913}, DOI={10.1007/BF02684768}, abstractNote={Iff is a C1 + ɛ diffeomorphism of a compact manifold M, we prove the existence of stable manifolds, almost everywhere with respect to everyf-invariant probability measure on M. These stable manifolds are smooth but do not in general constitute a continuous family. The proof of this stable manifold theorem (and similar results) is through the study of random matrix products (multiplicative ergodic theorem) and perturbation of such products.}, number={1}, journal={Publications Mathématiques de l’Institut des Hautes Études Scientifiques}, author={Ruelle, David}, year={1979}, month=dec, pages={27–58}, language={en} }

@inbook{ruelle1980, address={Berlin, Heidelberg}, series={Lecture Notes in Mathematics}, title={Stable manifolds for maps}, volume={819}, ISBN={978-3-540-10236-6}, url={http://link.springer.com/10.1007/BFb0087002}, DOI={10.1007/BFb0087002}, booktitle={Global Theory of Dynamical Systems}, publisher={Springer Berlin Heidelberg}, author={Ruelle, David and Shub, Michael}, editor={Nitecki, Zbigniew and Robinson, Clark}, year={1980}, pages={389–392}, collection={Lecture Notes in Mathematics}, language={en} }

@article{chemnitz2024, title={Characterizing Dynamical Stability of Stochastic Gradient Descent in Overparameterized Learning}, url={http://arxiv.org/abs/2407.20209}, DOI={10.48550/arXiv.2407.20209}, abstractNote={For overparameterized optimization tasks, such as the ones found in modern machine learning, global minima are generally not unique. In order to understand generalization in these settings, it is vital to study to which minimum an optimization algorithm converges. The possibility of having minima that are unstable under the dynamics imposed by the optimization algorithm limits the potential minima that the algorithm can find. In this paper, we characterize the global minima that are dynamically stable/unstable for both deterministic and stochastic gradient descent (SGD). In particular, we introduce a characteristic Lyapunov exponent which depends on the local dynamics around a global minimum and rigorously prove that the sign of this Lyapunov exponent determines whether SGD can accumulate at the respective global minimum.}, note={arXiv:2407.20209}, number={arXiv:2407.20209}, publisher={arXiv}, author={Chemnitz, Dennis and Engel, Maximilian}, year={2024}, month=sep }

@book{krantz2002, address={Boston, MA}, title={A Primer of Real Analytic Functions}, ISBN={978-1-4612-6412-5}, url={http://link.springer.com/10.1007/978-0-8176-8134-0}, DOI={10.1007/978-0-8176-8134-0}, publisher={Birkhäuser Boston}, author={Krantz, Steven G. and Parks, Harold R.}, year={2002}, language={en} }

@book{bogachev2007, address={Berlin, Heidelberg}, title={Measure Theory}, ISBN={978-3-540-34513-8}, url={http://link.springer.com/10.1007/978-3-540-34514-5}, DOI={10.1007/978-3-540-34514-5}, publisher={Springer Berlin Heidelberg}, author={Bogachev, Vladimir I.}, year={2007}, language={en} }

@book{nesterov2018, address={Cham}, series={Springer Optimization and Its Applications}, title={Lectures on Convex Optimization}, volume={137}, ISBN={978-3-319-91577-7}, url={http://link.springer.com/10.1007/978-3-319-91578-4}, DOI={10.1007/978-3-319-91578-4}, publisher={Springer International Publishing}, author={Nesterov, Yurii}, year={2018}, collection={Springer Optimization and Its Applications} }

@article{lyapunov,
author = {A. M. Lyapunov},
title = {The general problem of the stability of motion},
journal = {International Journal of Control},
volume = {55},
number = {3},
pages = {531-534},
year = {1992},
publisher = {Taylor & Francis},
doi = {10.1080/00207179208934253}
}

@article{lecun1998,
  title={Gradient-based learning applied to document recognition},
  author={LeCun, Yann and Bottou, L{\'e}on and Bengio, Yoshua and Haffner, Patrick},
  journal={Proceedings of the IEEE},
  volume={86},
  number={11},
  pages={2278--2324},
  year={1998},
  publisher={Ieee}
}

@book{hale1991,
    AUTHOR = {Hale, Jack K. and Koçak, H\"useyin},
     TITLE = {Dynamics and bifurcations},
    SERIES = {Texts in Applied Mathematics},
    VOLUME = {3},
 PUBLISHER = {Springer-Verlag, New York},
      YEAR = {1991},
     PAGES = {xiv+568},
      ISBN = {0-387-97141-6},
   MRCLASS = {58Fxx (34-01 34C23 58-01 58F14)},
  MRNUMBER = {1138981},
MRREVIEWER = {D.\ R. J. Chillingworth},
       DOI = {10.1007/978-1-4612-4426-4},
       URL = {https://doi.org/10.1007/978-1-4612-4426-4},
}

@article{rosca2023, 
  title={On a continuous time model of gradient descent dynamics and instability in deep learning}, 
  rights={Creative Commons Attribution 4.0 International}, url={https://arxiv.org/abs/2302.01952}, 
  DOI={10.48550/ARXIV.2302.01952}, 
  abstractNote={The recipe behind the success of deep learning has been the combination of neural networks and gradient-based optimization. Understanding the behavior of gradient descent however, and particularly its instability, has lagged behind its empirical success. To add to the theoretical tools available to study gradient descent we propose the principal flow (PF), a continuous time flow that approximates gradient descent dynamics. To our knowledge, the PF is the only continuous flow that captures the divergent and oscillatory behaviors of gradient descent, including escaping local minima and saddle points. Through its dependence on the eigendecomposition of the Hessian the PF sheds light on the recently observed edge of stability phenomena in deep learning. Using our new understanding of instability we propose a learning rate adaptation method which enables us to control the trade-off between training stability and test set evaluation performance.}, 
  publisher={arXiv}, 
  author={Rosca, Mihaela and Wu, Yan and Qin, Chongli and Dherin, Benoit}, 
  year={2023}, 
  month=feb 
}

@book{hairer2010,
    AUTHOR = {Hairer, Ernst and Lubich, Christian and Wanner, Gerhard},
     TITLE = {Geometric numerical integration},
    SERIES = {Springer Series in Computational Mathematics},
    VOLUME = {31},
      NOTE = {Structure-preserving algorithms for ordinary differential
              equations,
              Reprint of the second (2006) edition},
 PUBLISHER = {Springer, Heidelberg},
      YEAR = {2010},
     PAGES = {xviii+644},
      ISBN = {978-3-642-05157-9},
   MRCLASS = {65P10},
  MRNUMBER = {2840298},
}

@book{helmke1994,
    AUTHOR = {Helmke, Uwe and Moore, John B.},
     TITLE = {Optimization and dynamical systems},
    SERIES = {Communications and Control Engineering Series},
      NOTE = {With a foreword by R. Brockett},
 PUBLISHER = {Springer-Verlag London, Ltd., London},
      YEAR = {1994},
     PAGES = {xiv+391},
      ISBN = {0-387-19857-1},
   MRCLASS = {49-02 (58E25 58F40 90-02 93-02)},
  MRNUMBER = {1299725},
MRREVIEWER = {L.\ E.\ Faybusovich},
       DOI = {10.1007/978-1-4471-3467-1},
       URL = {https://doi.org/10.1007/978-1-4471-3467-1},
}

@article{brockett1991,
    AUTHOR = {Brockett, R. W.},
     TITLE = {Dynamical systems that sort lists, diagonalize matrices, and
              solve linear programming problems},
   JOURNAL = {Linear Algebra Appl.},
  FJOURNAL = {Linear Algebra and its Applications},
    VOLUME = {146},
      YEAR = {1991},
     PAGES = {79--91},
      ISSN = {0024-3795,1873-1856},
   MRCLASS = {90C05 (15A21 34A34 58F07 65K05)},
  MRNUMBER = {1083465},
       DOI = {10.1016/0024-3795(91)90021-N},
       URL = {https://doi.org/10.1016/0024-3795(91)90021-N},
}

@incollection{bloch1990,
    AUTHOR = {Bloch, A. M.},
     TITLE = {Steepest descent, linear programming, and {H}amiltonian flows},
 BOOKTITLE = {Mathematical developments arising from linear programming
              ({B}runswick, {ME}, 1988)},
    SERIES = {Contemp. Math.},
    VOLUME = {114},
     PAGES = {77--88},
 PUBLISHER = {Amer. Math. Soc., Providence, RI},
      YEAR = {1990},
      ISBN = {0-8218-5121-7},
   MRCLASS = {58F07 (34A99 58F05 65K05 90C05)},
  MRNUMBER = {1097866},
MRREVIEWER = {C.\ Tomei},
       DOI = {10.1090/conm/114/1097866},
       URL = {https://doi.org/10.1090/conm/114/1097866},
}

@article{mohtashami2022, 
	title={Special Properties of Gradient Descent with Large Learning Rates}, 
	rights={arXiv.org perpetual, non-exclusive license}, url={https://arxiv.org/abs/2205.15142}, 
	DOI={10.48550/ARXIV.2205.15142}, 
	abstractNote={When training neural networks, it has been widely observed that a large step size is essential in stochastic gradient descent (SGD) for obtaining superior models. However, the effect of large step sizes on the success of SGD is not well understood theoretically. Several previous works have attributed this success to the stochastic noise present in SGD. However, we show through a novel set of experiments that the stochastic noise is not sufficient to explain good non-convex training, and that instead the effect of a large learning rate itself is essential for obtaining best performance.We demonstrate the same effects also in the noise-less case, i.e. for full-batch GD. We formally prove that GD with large step size -- on certain non-convex function classes -- follows a different trajectory than GD with a small step size, which can lead to convergence to a global minimum instead of a local one. Our settings provide a framework for future analysis which allows comparing algorithms based on behaviors that can not be observed in the traditional settings.}, 
	publisher={arXiv}, 
	author={Mohtashami, Amirkeivan and Jaggi, Martin and Stich, Sebastian}, 
	year={2022}, 
	month=may 
}

@article{li2020, 
	title={Towards Explaining the Regularization Effect of Initial Large Learning Rate in Training Neural Networks}, 
	url={http://arxiv.org/abs/1907.04595}, 
	abstractNote={Stochastic gradient descent with a large initial learning rate is widely used for training modern neural net architectures. Although a small initial learning rate allows for faster training and better test performance initially, the large learning rate achieves better generalization soon after the learning rate is annealed. Towards explaining this phenomenon, we devise a setting in which we can prove that a two layer network trained with large initial learning rate and annealing provably generalizes better than the same network trained with a small learning rate from the start. The key insight in our analysis is that the order of learning different types of patterns is crucial: because the small learning rate model first memorizes easy-to-generalize, hard-to-fit patterns, it generalizes worse on hard-to-generalize, easier-to-fit patterns than its large learning rate counterpart. This concept translates to a larger-scale setting: we demonstrate that one can add a small patch to CIFAR-10 images that is immediately memorizable by a model with small initial learning rate, but ignored by the model with large learning rate until after annealing. Our experiments show that this causes the small learning rate model’s accuracy on unmodified images to suffer, as it relies too much on the patch early on.}, 
	note={arXiv:1907.04595 [cs, stat]}, 
	number={arXiv:1907.04595}, 
	publisher={arXiv}, 
	author={Li, Yuanzhi and Wei, Colin and Ma, Tengyu}, 
	year={2020}, 
	month=apr 
}

@book{shub1987, address={New York, NY}, title={Global Stability of Dynamical Systems}, ISBN={978-1-4419-3079-8}, url={http://link.springer.com/10.1007/978-1-4757-1947-5}, DOI={10.1007/978-1-4757-1947-5}, publisher={Springer New York}, author={Shub, Michael}, year={1987}, language={en} }

@book{lee2012, address={New York, NY}, series={Graduate Texts in Mathematics}, title={Introduction to Smooth Manifolds}, volume={218}, ISBN={978-1-4419-9981-8}, url={https://link.springer.com/10.1007/978-1-4419-9982-5}, DOI={10.1007/978-1-4419-9982-5}, publisher={Springer New York}, author={Lee, John M.}, year={2012}, collection={Graduate Texts in Mathematics}, language={en} }

@book{lee2011, address={New York, NY}, series={Graduate Texts in Mathematics}, title={Introduction to Topological Manifolds}, volume={202}, ISBN={978-1-4419-7939-1}, url={https://link.springer.com/10.1007/978-1-4419-7940-7}, DOI={10.1007/978-1-4419-7940-7}, publisher={Springer New York}, author={Lee, John M.}, year={2011}, collection={Graduate Texts in Mathematics}, language={en} }

@inproceedings{ahn2022,
  title={Understanding the unstable convergence of gradient descent},
  author={Ahn, Kwangjun and Zhang, Jingzhao and Sra, Suvrit},
  booktitle={International Conference on Machine Learning},
  pages={247--257},
  year={2022},
  organization={PMLR}
}

@book{kato1995, 
  address={Berlin}, 
  series={Classics in mathematics}, 
  title={Perturbation theory for linear operators}, 
  ISBN={978-3-540-58661-6}, 
  callNumber={QA329.2 .K37 1995}, 
  publisher={Springer}, 
  author={Katō, Tosio}, 
  year={1995}, 
  collection={Classics in mathematics} 
}

@article{cooper2018,
  title={The loss landscape of overparameterized neural networks},
  author={Cooper, Yaim},
  journal={arXiv preprint arXiv:1804.10200},
  year={2018}
}

@article{kawaguchi2016,
  title={Deep learning without poor local minima},
  author={Kawaguchi, Kenji},
  journal={Advances in neural information processing systems},
  volume={29},
  year={2016}
}

@inproceedings{arora2022,
  title={Understanding gradient descent on the edge of stability in deep learning},
  author={Arora, Sanjeev and Li, Zhiyuan and Panigrahi, Abhishek},
  booktitle={International Conference on Machine Learning},
  pages={948--1024},
  year={2022},
  organization={PMLR}
}

@article{lee2019,
  title={First-order methods almost always avoid strict saddle points},
  author={Lee, Jason D and Panageas, Ioannis and Piliouras, Georgios and Simchowitz, Max and Jordan, Michael I and Recht, Benjamin},
  journal={Mathematical programming},
  volume={176},
  pages={311--337},
  year={2019},
  publisher={Springer}
}

@inproceedings{cohen2022,
title={Gradient Descent on Neural Networks Typically Occurs at the Edge of Stability},
author={Jeremy Cohen and Simran Kaur and Yuanzhi Li and J Zico Kolter and Ameet Talwalkar},
booktitle={International Conference on Learning Representations},
year={2021},
url={https://openreview.net/forum?id=jh-rTtvkGeM}
}

@inproceedings{cohen2022b,
title={Adaptive Gradient Methods at the Edge of Stability},
author={Jeremy Cohen and Behrooz Ghorbani and Shankar Krishnan and Naman Agarwal and Sourabh Medapati and Michal Badura and Daniel Suo and Zachary Nado and George E. Dahl and Justin Gilmer},
booktitle={NeurIPS 2023 Workshop Heavy Tails in Machine Learning},
year={2023},
url={https://openreview.net/forum?id=dHGNgkUcGd}
}

@inproceedings{li2022SGD,
title={What Happens after {SGD} Reaches Zero Loss? --A Mathematical Framework},
author={Zhiyuan Li and Tianhao Wang and Sanjeev Arora},
booktitle={International Conference on Learning Representations},
year={2022},
url={https://openreview.net/forum?id=siCt4xZn5Ve}
}

@article{baldi1989, title={Neural networks and principal component analysis: Learning from examples without local minima}, volume={2}, ISSN={08936080}, DOI={10.1016/0893-6080(89)90014-2}, number={1}, journal={Neural Networks}, author={Baldi, Pierre and Hornik, Kurt}, year={1989}, month=jan, pages={53–58}, language={en} }

@inproceedings{trager2020,
title={Pure and Spurious Critical Points: a Geometric Study of Linear Networks},
author={Matthew Trager and Kathlén Kohn and Joan Bruna},
booktitle={International Conference on Learning Representations},
year={2020},
url={https://openreview.net/forum?id=rkgOlCVYvB}
}

@article{nguegnang2021, 
	title={Convergence of gradient descent for learning linear neural networks}, 
	url={http://arxiv.org/abs/2108.02040}, 
	author={Nguegnang, Gabin Maxime and Rauhut, Holger and Terstiege, Ulrich}, 
	year={2021}
}

@article{bah2020,
  title={Learning deep linear neural networks: Riemannian gradient flows and convergence to global minimizers},
  author={Bah, Bubacarr and Rauhut, Holger and Terstiege, Ulrich and Westdickenberg, Michael},
  journal={Information and Inference: A Journal of the IMA},
  volume={11},
  number={1},
  pages={307--353},
  year={2022},
  publisher={Oxford University Press}
}

\end{document}